\documentclass{article}

\PassOptionsToPackage{numbers, compress}{natbib}
\usepackage[preprint]{neurips_2023}

\usepackage{array, makecell}
\usepackage{url}            % simple URL typesetting
\usepackage{booktabs}       % professional-quality tables
\usepackage{nicefrac}       % compact symbols for 1/2, etc.
\usepackage{microtype}      % microtypography
\usepackage{latexsym}
\usepackage{graphicx,grffile} % the latter is for inserting image name with
\usepackage{mathrsfs} % for command \mathscr
\usepackage{amsmath,amssymb,amsthm}  % math packages
\usepackage[american]{babel}
\usepackage[colorlinks,citecolor=blue,urlcolor=blue]{hyperref}
\usepackage{bm}
\usepackage{dsfont}
\usepackage{xcolor}
\usepackage[T1]{fontenc}
\usepackage[utf8]{inputenc}
\usepackage[ruled]{algorithm2e}
\usepackage{float}
\usepackage{mathtools}
\usepackage{pdfpages}  % for inserting pdf
\usepackage{color}
\usepackage[capitalize,noabbrev]{cleveref}
\usepackage{custom_style}
\usepackage{mathrsfs} % for command \mathscr
\usepackage{bm}
\usepackage{natbib}
\usepackage{wrapfig}

\usepackage{tabu}
\usepackage{subfig}

\usepackage{lineno}
\usepackage{thm-restate}  % restate theorem

\newcommand{\minisection}[1]{\vspace{2mm}\noindent{\textbf{#1}}}

\theoremstyle{plain}
\newtheorem{theorem}{Theorem}[section]
\newtheorem{proposition}[theorem]{Proposition}

\theoremstyle{definition}
\newtheorem{definition}[theorem]{Definition}

\theoremstyle{remark}

\SetKw{kwInput}{Input:}
\SetKw{kwOutput}{Output:}

\title{Flow Matching in Latent Space}

\author{%
  Quan Dao \thanks{equal contribution} \\
  VinAI Research\\
  \texttt{v.quandm7@vinai.io} \\
  \And
  Hao Phung \footnotemark[1] \\
  VinAI Research\\
  \texttt{v.haopt12@vinai.io} \\
  \And
  Binh Nguyen \\
  Department of Mathematics, \\
  National University of Singapore\\
    \texttt{tuanbinhs@gmail.com}
  \AND
  Anh Tran \\
  VinAI Research\\
  \texttt{v.anhtt152@vinai.io} \\
}

\begin{document}
\maketitle

\begin{abstract}
Flow matching is a recent framework to train generative models that exhibits impressive empirical performance while being relatively easier to train compared with diffusion-based models.
 Despite its advantageous properties, prior methods still face the challenges of expensive computing and a large number of function evaluations of off-the-shelf solvers in the pixel space. Furthermore, although latent-based generative methods have shown great success in recent years, this particular model type remains underexplored in this area. In this work, we propose to apply flow matching in the latent spaces of pretrained autoencoders, which offers improved computational efficiency and scalability for high-resolution image synthesis. This enables flow-matching training on constrained computational resources while maintaining their quality and flexibility.
 Additionally, our work stands as a pioneering contribution in the integration of various conditions into flow matching for conditional generation tasks, including label-conditioned image generation, image inpainting, and semantic-to-image generation. Through extensive experiments, our approach demonstrates its effectiveness in both quantitative and qualitative results on various datasets, such as CelebA-HQ, FFHQ, LSUN Church \& Bedroom, and ImageNet. We also provide a theoretical control of the Wasserstein-2 distance between the reconstructed latent flow distribution and true data distribution, showing it is upper-bounded by the latent flow matching objective. Our code will be available at \href{https://github.com/VinAIResearch/LFM.git}{https://github.com/VinAIResearch/LFM.git}.
\end{abstract}

\section{Introduction}
Generative AI is one of the most active fields of AI research in recent years, with a rapid growth of technologies and a blooming of applications. In the context of computer vision, Generative AI not only captures data distribution of a target domain, allowing producing novel data without or with conditional inputs such as text, sketches, semantic maps, or images, but it also brings in various downstream applications such as image restoration, image translation, image segmentation, and more.
While GANs paved the way for high-fidelity image generation \cite{brock2018large}, diffusion models \cite{dhariwal2021diffusion,huang2021variational,kingma2021variational,song2021maximum,xiao2021tackling} have taken the lead with better mode coverage and excellent support on different conditional inputs.

However, diffusion models are not yet the perfect since it requires a long sampling time. A line of methods \cite{dockhorn2022genie,jolicoeur2021gotta,lu2022dpm,song2021denoising,zhang2023fast} is dedicated to improving sampling efficiency to address their slow training and testing speed. Despite some successes with these methods, diffusion models still suffer from other issues related to slow convergence and sub-optimal probability paths, which can hinder the performance of existing methods on large-scale datasets. Besides fixing the flaws of diffusion models, another promising research direction is to design a more efficient AI generation technique that inherits their advantages but does not suffer from mentioned problems.

Flow matching framework \cite{albergo2022building,lipman2023flow,liu2023flow}, a newly introduced line of research, show positive initial results and promise to become a better alternative for diffusion models. 
The essence of this method lies in learning an ordinary differential equation that follows a path from the source distribution to the target distribution. It has been pointed out that ODE-based generative models have favorable properties compared to SDE-based ones in diffusion models, most notably easier to train and faster to sample, due to lower curvature in the generative trajectories \cite{lee2023minimizing,liu2023flow}. Particularly, in \cite{lipman2023flow}, flow matching models beat the standard diffusion method DDPM on unconditional image generation tasks in terms of negative log-likelihood, sample quality, and inference speed.
Despite its initial promising results, flow matching (FM) research is still in the infancy stage. In this paper, we aim to boost its development by adapting the key successful features of diffusion models into these systems. 

First, although state-of-the-art diffusion models flourish via executing the diffusion process in the latent space, all FM models still perform path matching in pixel space. We propose to apply flow matching in the latent spaces of pretrained autoencoders, which offers improved computational efficiency and scalability for high-resolution image synthesis.

Second, one of the key successes of diffusion models is the versatile support of different conditional input types. Prior FM models are developed only for unconditional tasks. We redesign the velocity field network to support conditional inputs, such as class labels, segmentation masks, or images, and demonstrate this ability via different downstream tasks. To our understanding, our work are also the first to integrate classifier-free guidance in the sampling process of flow matching models.

Finally, we demonstrate that the Wasserstein-2 distance between the reconstructed latent flow distribution and true data distribution is upper-bounded by the latent flow matching objective. Additionally, our bound emphasizes the importance of choosing good backbone autoencoder architectures.

In summary, our paper offers the following contributions:

\begin{itemize}
\item To the best of our knowledge, our work is the first to introduce flow matching in latent space for high-resolution datasets, resulting in faster training time and improved computational efficiency.
\item Our work is also the first to integrate conditional inputs into flow matching models and support class conditional image generation with classifier-free velocity field guidance. We demonstrate this ability via various downstream applications, including label-conditioned image generation, image inpainting, and mask-to-image generation.
\item We provide a theoretical analysis (Wasserstein bound) on the tradeoff of performing flow matching of latent representations.
\item Extensive experiments on various tasks demonstrate our method's superiority in quantitative and qualitative results, narrowing the gap with state-of-the-art diffusion methods. Our approach is expected to facilitate the development of flow matching compared to other generative models in real-world applications.
  
\end{itemize}

\section{Related works}
\minisection{Diffusion models} \cite{sohl2015deep} is an emergent type of generative framework that achieves unprecedented performance in multiple fields, including image generation \cite{ho2020denoising, nichol2021improved, song2021denoising, dhariwal2021diffusion, Peebles2022DiT}, video generation \cite{ho2022video,ho2022imagen}, 3D generation \cite{luo2021diffusion}, and language generation \cite{dieleman2023language,li2022diffusion}. This type of model can be viewed as a score-based model \cite{song2019generative, song2020improved, song2020score, vincent2011connection}, as it relies on Stochastic Differential Equations (SDEs) to find trajectories from a tractable distribution (such as simple Gaussian noise) to the data distribution.
However, these models typically require a significant number of function evaluations, ranging from hundreds to thousands, to generate a sample. This process can be slow and computationally inefficient for real-world applications. Hence, recent research has focused on improving the sampling efficiency by either modifying the diffusion process \cite{song2021denoising, zhang2023fast, karras2022elucidating,song2023consistency} or using generative adversarial training \cite{xiao2021tackling,phung2023wavediff}. Another direction is to have instead shifted the focus to apply diffusion generative model in latent space to reduce the compute cost, benefiting from the compact dimension of latent representations learned by an autoencoder  \cite{rombach2022high, vahdat2021scorebased, sinha2021d2c}. Still, diffusion models have some drawbacks related to training convergence and sub-optimal trajectories, which might adversely affect their training time and overall performance.  Motivating from \cite{rombach2022high, vahdat2021scorebased, sinha2021d2c}, our work is also performed in the latent space of a pre-trained autoencoder. To the best of our knowledge, we are the first to leverage latent representation for flow-matching, specifically targeting high-resolution image generation.

\minisection{Continuous normalizing flows (CNFs)} \cite{chen2018neural,grathwohl2018ffjord}, unlike SDE-based methods, learn a deterministic mapping from the source distribution to the target distribution via neural Ordinary Differential Equation.
However, previous framework for the training of CNFs are computational prohibitive and hard to scale up due to the need of solving ODEs at each iteration.
As such, no existing methods are on par with the SDE-based models in the literature in terms of performance.
Inspired by the per-sample optimization approach of score-based diffusion models, flow matching (FM) \cite{albergo2022building,lipman2023flow,liu2023flow} was recently introduced as a simulation-free technique for training CNFs.
It offers the utilization of flexible probability paths to efficiently training CNFs that transport noise samples into data samples.
Moreover, based on optimal transport theory, \cite{lipman2023flow,liu2023flow} argues that using training a constant velocity field result in simpler training objective, compared to the high-curvature probability path in diffusion models.
The attractive properties of FM lead to a current trend in extending this method to more task-specific applications \cite{chen2023riemannian,li2023self,wu2022fast} and faster sampling \cite{lee2023minimizing,pooladian2023multisample,zheng2023improved}.

However, a main drawback of current works on FM is that they are not yet ready for high-resolution image synthesis.

Our approach represents one of the pioneering attempts to thoroughly integrate and study latent representations for flow-matching models, with the aim of enhancing both scalability and performance.

\minisection{Class-conditional generation} has attained significant advancements in recent years, primarily driven by the success of diffusion models \cite{dhariwal2021diffusion, ho2022cdm,zheng2023eds}.
ADM \cite{dhariwal2021diffusion}, for instance, utilizes the gradient of a pre-trained classifier w.r.t. the input sample to update the denoising process. Notably, \cite{ho2022classifier} introduced a simple yet elegant technique known as classifier-free guidance. This approach effectively balances the trade-off between the quality and diversity of generated samples without requiring a pretrained classifier. Besides, such a technique has made significant contributions to the advancement of conditional generation in various domains, such as text-to-image generation \cite{ramesh2022hierarchical,saharia2022photorealistic} and text-to-video generation \cite{ho2022imagen,singer2022make}. However, the application of flow-based models in class-conditional generation has not yet been explored. 

To tackle this gap, we incorporate classifier-free guidance from \cite{ho2022classifier} for latent flow matching, which plays a pivotal role in enhancing the performance of conditional models. However, it is important to note that while we utilize the same technique, it differs from traditional generative frameworks in terms of the target being estimated. In flow matching, the sampling process is driven by velocity rather than noise, resulting in a distinct approach to conditional generation.

\section{Background}

Given access to empirical observations of data distribution $\bx_0 \sim p_0$
and noise distribution $\bx_1 \sim p_1$ (which usually is Gaussian), the goal
of flow matching framework is to estimate a coupling $\pi(p_0, p_1)$ that
describes the evolution between those two distributions.
This objective could be formulated as solving an ordinary differential equation:
\begin{align}
  \label{eq:ode-lagrangian}
  \mathrm{d}\bx_t = v(\bx_t, t) \mathrm{d}t,
\end{align}
on time $t \in [0,1]$, where the velocity $v:\bbR^d \times [0, 1] \to \bbR^d$
is set to drive the flow from $p_0$ to $p_1$.

We can parameterize the drift (velocity) by $v_{\theta}(x_t, t)$ with an expressive learner (a neural net, for example) and estimate $\theta$ by solving a simple least square regression problem:
\begin{equation}
  \label{eq:fm-obj}  
  \hat{\theta} = \argmin_{\theta} \bbE_{t, \bx_t} \left[\norm{v(\bx_t, t) - v_{\theta}(\bx_t, t) }^2_2 \right].
\end{equation}
With this estimation, we can do backward sampling by taking
$\hat{\bx}_0 = \int_1^0 v_{\theta}(\bx_t, t) \mathrm{d}t$
since we have access to the noise $\bx_1 \sim p_1$, and solve the integration with numerical integrators.
Formally, the ODE in \eqref{eq:ode-lagrangian} is called the Lagrangian flow, which describes the dynamic of point clouds.
We have an equivalent view in Eulerian sense -- a continuity equation that describes the dynamic of the measure $p_t$ \cite{ambrosio2005gradient}:
\begin{equation}
  \label{eq:ode-eulerian}
  \partial_t p_t  = -\mathrm{div}(v(\bx, t) p_t),
\end{equation}
where \texttt{div} is the divergence operator.

Equation \Cref{eq:fm-obj} allows for the inclusion of different options for $v(\bx, t)$, which underscores the flexibility nature of the flow matching framework.
In the following, we provide a brief overview of two widely-used variations of $v(\bx, t)$: the probability flow ODE and the constant velocity ODE.

\paragraph{Probability flow ODE}
Introduced in \cite{song2020score}, the velocity in probability flow ODE has the form
\begin{equation}
  \label{eq:velocity-probaflow}
  v_t(\bx_t, t) = f(\bx_t, t) - \frac{g_t^2}{2} \nabla \log p_t,
\end{equation}
where $\nabla \log p_t$ is the score function, $f(\bx_t, t)$ and $g_t$ are the drift and diffusion coefficients of the equivalent form of the forward generative diffusion process, a stochastic differential equation (SDE)
\begin{equation}
  \label{eq:generative-sde}
  \mathrm{d}\bx_t = f(\bx_t, t)\mathrm{d}t + g_t\mathrm{d}\bw
\end{equation}
with $\bw$ the standard Wiener process (or Brownian motion).
The flow matching loss then can be recasted as score matching loss \cite{hyvarinen2005estimation} to estimate the score $\nabla \log p_t$, as suggested in \cite{lee2023minimizing,zheng2023improved}.
A common choice of the path $\bx_t$ is the Variance Preserving (VP) path:
\begin{equation}
  \label{eq:vp-path}
  \bx_t \egaldef \alpha_t \bx_0 + (1-\alpha_t^2)^{\frac{1}{2}} \ \bx_1 , \text{where } \alpha_t = e^{-\frac{1}{2}\int_0^{t} \beta(s)\mathrm{d}s}.
\end{equation}
%
% \binh{to check back the VP ODE formulation}
%
It is shown in \cite{karras2022elucidating} that using ODE sampler of the form \eqref{eq:velocity-probaflow} with path \eqref{eq:vp-path} can reduce sampling costs compared to sampling with discretization of the diffusion SDE in Equation~\eqref{eq:generative-sde}.

\paragraph{Constant velocity ODE}
Liu et al. \cite{liu2023flow} pointed out that using a non-linear interpolation for $\bx_t$ in the VP probability flow ODE in \eqref{eq:vp-path} can result in an unnecessary increase in the curvature of generative trajectories.
As a consequence, this can lead to reduced efficiency in the training and sampling of generative trajectories.
Instead, they advocate the use of constant velocity ODE, where the path $\bx_t \egaldef (1 - t)\bx_0 + t\bx_1$ is the linear interpolation between $\bx_0$ and $\bx_1$.
This means the velocity drives the flow following the direction $v_t = \bx_1 - \bx_0$, and the flow matching loss in \Cref{eq:fm-obj} becomes:
\begin{equation}
  \label{eq:fm-obj-constant}  
  \hat{\theta} = \argmin_{\theta} \bbE_{t, \bx_t} \left[\norm{\bx_1 - \bx_0 - v_\theta(\bx_t, t) }^2_2 \right].
\end{equation}
We thus develop our method using the same framework: incorporating the use of linear interpolation for training and the ODE presented in Equation~\eqref{eq:ode-lagrangian} for sampling.

\section{Methodology}

\begin{figure}[!ht]
    \centering
    \includegraphics[width=\linewidth]{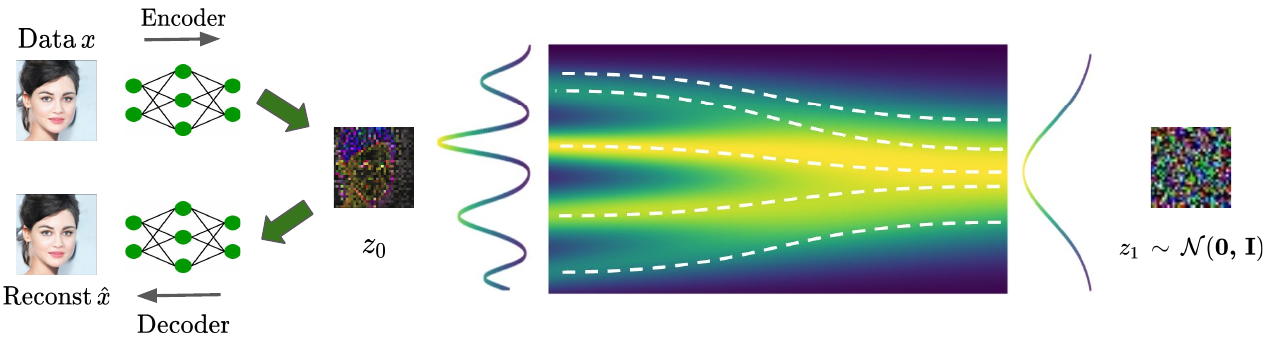}
    % \vspace{-1mm}
    \caption{In our Latent Flow Matching framework's training procedure, the input data $\bx$ is encoded to produce the latent representation $\bz_0$. A latent flow network is then trained to predict the velocity of the transformation from a standard normal distribution $p(\bz_1) = \cN(0, \bI)$ to the target latent distribution $p(\bz_0)$. During the sampling process, random noise $\bz_1$ is drawn from $p(\bz_1)$, and the trained latent network is used to predict the velocity towards the target latent distribution $p(\bz_0)$ through numerical integration. Finally, $\bz_0$ is decoded to generate the corresponding image.}
    \label{fig:framework}
\end{figure}

In this section, we first introduce the training and sampling procedure of our proposed framework, then present the classifier-free velocity field for class-conditional image generation.

\subsection{Training and sampling procedure}
We give a summary of our framework in \cref{fig:framework}.
% \paragraph{Training}
% %
Given an input sample $\bx_0 \sim p_0$, we encoded it into latent code $\bz_0 = \cE(\bx_0) \in  \mathbb{R}^{d/h}$ via a pretrained VAE encoder with KL regularization \cite{kingma2013auto}, where $h$ is the downsampling rate of the encoder.
In this latent space, the goal is to estimate a probability path to traverse from a random noise $\bz_1 \sim \cN(0, \bI)$ to the source distribution of latent codes $\bz_0$.
The optimization of the velocity network can take advantage of the compact dimension of latent codes, where we use the same objective as in the vanilla flow matching with constant velocity in Equation~\cref{eq:fm-obj-constant}:

\begin{equation}
  \label{eq:lfm-obj-constant}  
  \hat{\theta} = \argmin_{\theta} \bbE_{t, \bz_t} \left[\norm{\bz_1 - \bz_0 - v_\theta\left(\bz_t, t\right) }^2_2 \right].
\end{equation}

To train the conditional latent flow matching, we pass the conditional information $\bc$ to the flow model along with $\bz_t$ as:

\begin{equation}
  \hat{\theta} = \argmin_{\theta} \bbE_{t, \bz_t} \left[\norm{\bz_1 - \bz_0 - v_\theta\left(\bz_t, \bc, t\right) }^2_2 \right].
  \label{eq:clfm-obj-constant}
\end{equation}

With the sampling process, we test different numerical ODE integration schemes, such as Euler or a variant of Runge-Kutta using fixed or adaptive-step.

Following \cite{karras2022elucidating}, we also create custom solvers by performing iterative sampling from discrete time intervals $\{t_i\}^N_{i=1}$ over $N$ finite steps that works with any numerical ODE solvers given a velocity estimator $v_\theta$.

Details on our selection of ODE solvers are shown in \cref{ssec:solvers}, where we observed that adaptive solver such as \texttt{dropri5} \cite{Dormand1980AFO} yields comparable results with Heun solver \cite{zheng2023eds}. Still, for the sake of simplicity, we opt to use the adaptive-step one as the default algorithm for all of our experiments.

Finally, the desired sample $\hat{\bz}$ is then decoded by a pretrained VAE decoder $\mathcal{D}$ to yield the output image $\hat{\bx} = \mathcal{D}(\hat{\bz})$.

\subsection{Conditional generation with classifier-free guidance for velocity field}
Given the significance of class information $\bc$ in GANs \cite{brock2018large,mirza2014conditional} and denoising diffusion models \cite{dhariwal2021diffusion,ho2022cdm,ho2022classifier,zheng2023eds} for conditional image synthesis, it is reasonable to incorporate class label information into the ODE sampler of our flow matching framework.
In \cite{dhariwal2021diffusion}, the authors make use of a pre-trained classifier $p(\bc|\bx)$ to guide the sampling process.
However, this introduces additional complexity to the training pipeline due to the necessity of training an extra classifier.
Moreover, this classifier must be trained on noisy interpolation path $\bx_t$, making it generally hard to incorporate the pre-trained classifier seamlessly.
% %
As a result, we propose a formulation for classifier-free velocity field, inspired from \cite{ho2022classifier}.

First, we rewrite the generative SDE \eqref{eq:generative-sde} using the associated Fokker-Planck equation:
\begin{equation}
  \partial_t \tilde{p}_t = -\mathrm{div}(f(\bx_t, t)\tilde{p}_t) + \frac{g_t^2}{2} \Delta \tilde{p}_t , 
\end{equation}
where $\tilde{p}_t \egaldef p_t(\bx_t|\bx_0, \bc)$ is the conditional probability of a sample given its label $\bc$.
We can transform this into continuity equation of the form \eqref{eq:ode-eulerian} by
\begin{align}
  \partial_t \tilde{p_t} = -\mathrm{div} \left[f(\bx_t, t)p_t - \frac{g_t^2}{2} \frac{\nabla\tilde{p}_t}{\tilde{p}_t} \tilde{p}_t \right]
                        = -\mathrm{div}\left[\left(f(\bx_t, t) - \frac{g_t^2}{2} \nabla \log\tilde{p}_t\right)\tilde{p}_t \right].
\end{align}
This suggests the modified velocity vector has the form
\begin{align}
  \label{eq:equiv-velocity-probaflow}
  \tilde{v}(\bx_t, t) = f(\bx_t, t) - \frac{g_t^2}{2} \nabla \log \tilde{p}_t
                        = f(\bx_t, t) - \frac{g_t^2}{2} (\nabla \log p_t + \nabla_x \log p (\bc | \bx_t)),
\end{align}
since we can factorize $\tilde{p}_t \egaldef p_t(\bx_t|\bx_0, \bc) = Z p_t(\bx_t|\bx_0)p (\bc|\bx_t) $, where $Z$ is the normalizing constant that we can safely ignore after taking gradient of the log-likelihood.
%
% We can parameterize $p (y | \bx_t)$, the conditional probability of a classifier trained on a particular labeled dataset, by $p_{\omega} (y | \bx_t)$ the last softmax layer of a neural network, for example.
%
% Note that this equation is equivalent to transform the generative SDE into the probability flow ODE with an addition classifier guidance term:
% %
% \begin{equation}
%   d\bx_t = [f(\bx_t, t) - \frac{g_t^2}{2}(\nabla_x \log p_t + \nabla_x \log
%   p_{\omega} (y | \bx_t)) ] dt . 
% \end{equation}
%
Using the same argument as in \cite{kingma2021variational} for variance preserving ODE, with the constant velocity path $\bx_t = (1-t)\bx_0 + t\bx_1$ that follows $\cN(\bx_t; (1-t)\bx_0, t^2\bI)$, we have
\begin{equation*}
  \nabla_x \log p_t = \frac{-(\bx_t - (1-t)\bx_0))}{t^2}, % = \frac{-\bx_1}{t},
  \quad f(\bx_t, t) = \frac{-\bx_t}{1-t}, \text{ and} \quad
\frac{g_t^2}{2} = \dfrac{t}{1-t}. 
\end{equation*}
Plugging these into \eqref{eq:equiv-velocity-probaflow}, we get
\begin{equation}
  \tilde{v}(\bx_t, \bc, t) = (\bx_1 - \bx_0)
  - \left(\frac{t}{1-t}\right)\nabla_x \log p(\bc | \bx_t) = v(\bx_t, t) - \left(\frac{t}{1-t}\right)\nabla_x \log p(\bc | \bx_t).
\end{equation}
%
% where we use the fact that with the constant velocity ODE, $\mu_t(\bx_0) = (1 - t)\bx_0, \sigma_t(\bx_0) = t$.
%
This leads to the classifier-guided velocity field, similar to classifier-guided diffusion model in \cite{dhariwal2021diffusion}) as
\begin{equation}
  \label{eq:guided-velocity}  
  \tilde{v}(\bx_t, \bc, t) = v_{\theta}(\bx_t, t) - \gamma\left(\frac{t}{1-t}\right)\nabla_x \log p(\bc | \bx_t) , 
\end{equation}
where we add a scaling factor $\gamma > 0$ that controls the gradient strength.
%
% Eq.~\eqref{eq:guided-velocity} suggests that we can use a pretrained velocity network $v_{\theta}(\bx_t, t)$ and coupled it with a trained classifier $p_{\omega} (y | \bx)$ to make label-conditioned sampling.
%
% \paragraph{Classifier-free velocity}
% The derivation of classifier-guided velocity field suggests
%
%
Finally, to arrive at the classifier-free velocity field formulation, first we add $\bc$ as an input of the velocity field estimation network, denoted $v_{\theta}(\bx_t, \bc, t)$. Then, factorizing again with Bayes rule the term $\nabla_x \log p(\bc|\bx_t)$ in \eqref{eq:guided-velocity} with $\nabla_x \log p(\bx_t | \bx_0, \bc) - \nabla_x \log p(\bx_t | \bx_0)$. From \eqref{eq:guided-velocity}, we have:
\begin{align}
    \tilde{v}(\bx_t, \bc, t) &= v_{\theta}(\bx_t, t) - \gamma\left(\frac{t}{1-t}\right)\nabla_x \log p(\bc | \bx_t) \nonumber\\
     &= v_{\theta}(\bx_t, t) - \gamma\left(\frac{t}{1-t}\right)\left(\nabla_x \log p(\bx_t | \bx_0, \bc) - \nabla_x \log p(\bx_t | \bx_0)\right) \label{eq15}
\end{align}
Using \eqref{eq:velocity-probaflow}, we have $- v_t(\bx_t, \bc, t) + f(\bx_t, t) = \frac{g_t^2}{2} \nabla \log p_t(\bx_t | \bc)$. Since $\bx_t$ is fixed and $f, g$ are well-defined functions, we have:
\begin{align}
 \frac{g_t^2}{2}\left(\nabla \log p_t(\bx_t | \bc) - \nabla \log p_t(\bx_t | \bc = \emptyset)\right) &= - v_t(\bx_t, \bc, t) + f(\bx_t, t) + v_t(\bx_t, \bc = \emptyset, t) - f(\bx_t, t) \nonumber\\
 &= -v_t(\bx_t, \bc, t) + v_t(\bx_t, \bc = \emptyset, t) \label{eq16}
\end{align}\
Substitute \eqref{eq16} into \eqref{eq15}, we can approximate:
\begin{align}
  \tilde{v}_{\theta}(\bx_t, \bc, t) \approx v_{\theta}(\bx_t, t) + \gamma ( v_{\theta}(\bx_t, \bc, t) - v_{\theta}(\bx_t, t)) = \gamma v_{\theta}(\bx_t, \bc, t) + (1 - \gamma) v_{\theta}(\bx_t, \bc=\emptyset, t),
\end{align}
where $v_{\theta}(\bx_t, t)= v_{\theta}(\bx_t, \bc=\emptyset, t)$ denote the unconditional velocity field trained with $\bc$ set to the empty token with a certain probability, akin to the philosophy of dropout in training neural networks.
This enables joint training for the unconditional model and conditional model within a single neural network, which does not include integration of a pretrained classifier.

\subsection{Theoretical analysis: bounding estimation error of latent flow matching}
\label{ssec:theory}
In order to perform theoretical analysis on performance of latent flow matching, we follow the settings of \cite{camuto21towards,kumar2020implicit}.
We assume the decoder is deterministic and is parameterized by a function $g_{\tau}:\bbR^{d/h}\to\bbR^d$, the Gaussian encoder that parameterized the posterior $q^{\phi}(\bz \mid \bx) = \cN(\mu_{\phi}(\bx), \sigma^2_{\phi}(\bx) \bI)$.
In other words, we use diagonal Gaussian reparameterization \cite{kingma2013auto}:
\begin{equation}
\label{eq:vae-reparam}
  \bz = \mu_{\phi}(\bx) + \eta \circ \sigma_{\phi}(\bx) \quad \text{with } \eta \sim \cN(0, \bI),
\end{equation}
where $\circ$ denotes element-wise multiplication.
We then have the following theorem.
\begin{theorem}[Proof in \cref{sec:appendix-proof}]
  \label{thm:wasserstein-bound}
  Let $p_0: \bbR^d \to \bbR$ denote the ground truth distribution of data
  $\bx_0$, $\hat{p}_0$ the distribution of reconstructed sample
  $\hat{\bx}$.
  Let the latent code $\bz$ be encoded following diagonal Gaussian reparameterization trick in \eqref{eq:vae-reparam}.
  % $p^{\phi}(\bz)$ the latent representation of $p_0$ with a learned VAE,
  Let $v(\bz_t, t)$ be the solution of the following continuity equation
  \begin{equation*}
    \partial_t q_t^{\phi} = - \mathrm{div}(v(\bz_t, t) q_t^{\phi}), \quad\quad
    q_{t=1}^{\phi} = q^{\phi}_1,
  \end{equation*}
  and $\hat{v}(\hat{\bz}_t, t)$ be an approximation of $v(\bz_t, t)$ using the flow matching objective \eqref{eq:lfm-obj-constant}, or in other words, $\hat{v}(\hat{\bz}_t, t)$ is the velocity field that defines $\hat{q}_t$, a solution of
  \begin{equation}
    \partial_t \hat{q}_t = - \mathrm{div}(v(\hat{\bz}_t, t) \hat{q}_t), \quad\quad
    \hat{q}_{t=1} = q^{\phi}_1.
  \end{equation}
  Assume that
  \begin{enumerate}
    \item The estimated velocity $\hat{v}$ is continuously differentiable in $(\bz, t)$, and, Lipschitz in $\bz$ uniformily on $R^{d/h} \times [0, 1]$ with constant $\hat{L}$.
   \item The deterministic decoder $g_{\tau}$ is Lipschitz continuous in $\bz$ with constant $L_{g_{\tau}}$.
    \item The distance from the quantity induced by taking  $g_{\tau}(\bz_0)$, \ie a VAE reconstructed sample, with any real data point $\bx_0 \sim p_0$ is defined by a quantity $\Delta_{f_{\phi},g_{\tau}}(\bx_0)$. In other words, $g_{\tau}(\bz_0) = \bx_0 + \Delta_{f_{\phi},g_{\tau}}(\bx_0)$, and we assume that it has bounded $L_2$ norm.
  \end{enumerate}
  Then we have the following bound on the squared Wasserstein-2 distance of the
  reconstructed distribution $\hat{p}_0$ and the ground truth $p_0$:
  \begin{equation}
    \label{eq:wasserstein-bound}
    \cW_2^2(p_0, \hat{p}_0) \leq \norm{\Delta_{f_{\phi},g_{\tau}}(\bx_0)}^2+ L^2_{g_{\tau}} e^{1 +
      2\hat{L}} \int_0^1\int_{\bbR^{d/h}} \abs{v(\bz_t, t) - \hat{v}(\hat{\bz}_t, t)}^2 \mathrm{d}q_t^{\phi}\mathrm{d}t,
  \end{equation}
  where $\Delta_{f_{\phi},g_{\tau}}(\bx_0)$ are constants that depend only on $f_{\phi}$ and $g_{\tau}$, the stochastic encoder and decoder.
\end{theorem}

\Cref{thm:wasserstein-bound} shows that the minimizing flow matching objective \eqref{eq:fm-obj} in latent space controls the Wasserstein distance between the target density $p_0$ and the reconstructed density $\hat{p}_0$.
We note that the bounded quantity on the LHS of \eqref{eq:wasserstein-bound} coincides with Fréchet inception distance (FID) \cite{heusel2017gans}, one of the most popular metrics for ranking generative models. % in generative modeling.
This means latent flow matching is guaranteed to control this metric, given reasonable estimation of $\hat{v}(\bz_t, t)$.
Nonetheless, our theoretical analysis also suggests that the quality of matching samples in latent flow significantly depends on the constants that define the expressive capacities of the decoders and encoders.
This highlights the tradeoff between computational speed and the quality of generated samples, which has been empirically observed in prior research on generative modeling in latent space:
for example, both \cite{rombach2022high,vahdat2021scorebased} emphasize the critical role of utilizing good VAE architectures as backbone models for latent diffusion models.

\section{Experiments}

In our experiments, we use the pretrained VAE \cite{kingma2013auto} from Stable Diffusion \cite{rombach2022high}. The VAE encoder has a downsampling factor of 8 given an RGB pixel-based image $\bx \in \bbR^{h \times w \times 3}$, $\bz = \mathcal{E}(\bx)$ has
shape $\frac{h}{8} \times \frac{w}{8} \times 4$. All experiments are operated in the latent space.

\begin{figure}[t]
    \centering
    \includegraphics[width=\linewidth]{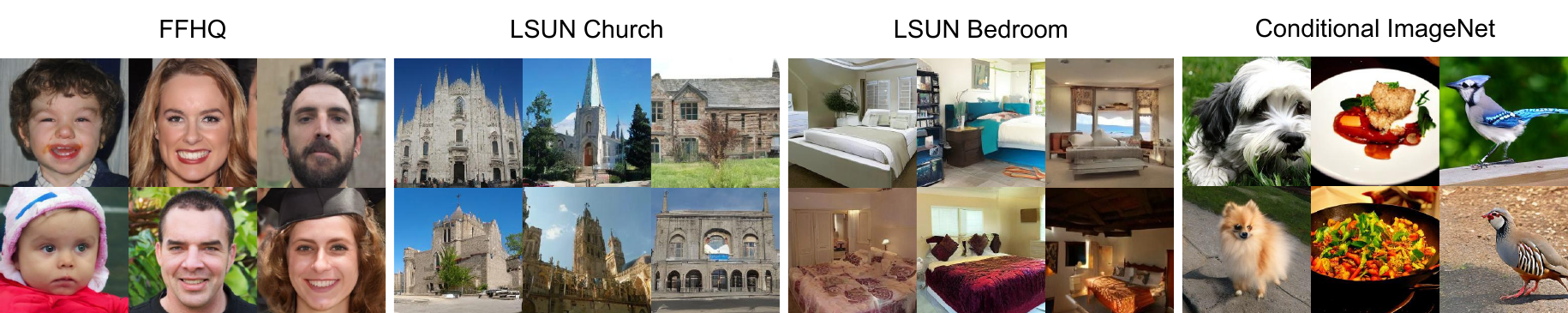}
    
    \vspace{-1mm}
    \caption{Qualitative examples of our unconditional generation on FFHQ and LSUN Church \& Bedroom and class conditional generation on ImageNet.}
    \label{fig:qualitatives}
    \vspace{-1mm}
\end{figure}

\begin{figure}
    \centering
    \includegraphics[width=.75\linewidth]{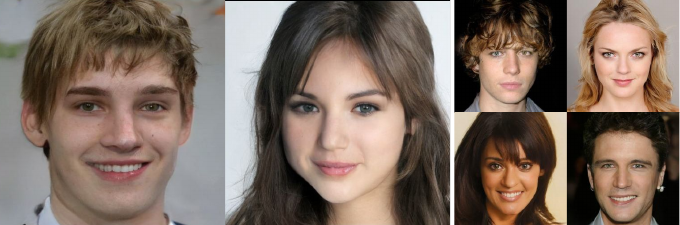}
    \vspace{-1mm}
    \caption{CelebA-HQ 512 and 256}
    \label{fig:celeb_samples}
    \vspace{-5mm}
\end{figure}

\subsection{Unconditional Generation}
We conducted experiments on various datasets, including CelebA-HQ \cite{karras2017progressive}, FFHQ \cite{karras2019style}, LSUN Bedrooms \cite{yu2015lsun}, and LSUN Church \cite{yu2015lsun}, with the same resolution of $256 \times 256$. We further benchmarked on CelebA-HQ with higher resolutions, including $512 \times 512$. Additionally, we perform our measurements on two distinct architectures: a CNN-based UNet (ADM) \cite{dhariwal2021diffusion} and a transformer network (DiT) \cite{Peebles2022DiT} for a comprehensive evaluation. To ensure a fair comparison, we employ two commonly used metrics for evaluating the performance of our proposed method: the Fréchet inception distance (FID) \cite{heusel2017gans} and Recall \cite{kynkaanniemi2019improved}. We also measure the average running time and the number of function evaluations for single image generation to assess the computational efficiency of our method. In \cref{tab:celeb256}, our method attains better FID scores than original flow matching FM on CelebA-HQ 256, at 5.82 and 5.26 for the CNN-based and transformer network architectures, respectively. On CelebA-HQ 512 and FFHQ, in \cref{fig:face256}, our method also shows superior performance than other diffusion counterparts. Furthermore, our method achieves comparable results with existing methods on the LSUN Church and Bedrooms datasets, as demonstrated in  \cref{fig:lsun}. In all experiments, it is clear that DiT produces better performance than ADM even though it is inherently designed for conditional image generation. Their qualitative results are provided at \cref{fig:qualitatives} and \cref{fig:celeb_samples}.

% Add some comments each table

\setlength{\tabcolsep}{3pt}
\begin{table}[t]
\caption{Results on CelebA-HQ and FFHQ datasets. (*) is measured on our machine}
\vspace{2mm}
\centering
\subfloat[CelebA-HQ $256$ \& $512$]{
\raisebox{.0\height} {
\resizebox{.459\linewidth}{!}{
% \begin{tabu} to 0.4\textwidth{ccc}
  \begin{tabu}{l c c c c}
    \toprule
    Model & NFE$\downarrow$ & FID$\downarrow$ & Recall$\uparrow$ & Time(s)$\downarrow$ \\
    \midrule
    \multicolumn{5}{c}{\textbf{CelebA-HQ 256}} \\
    \midrule
    Ours (ADM) & 85 & 5.82 & 0.41 & 3.42 \\
    Ours (DiT L/2) & 89 & 5.26 & 0.46 & 1.70 \\
    FM \cite{lipman2023flow}* &128  & 7.34 & - & -\\
    \midrule
    LDM \cite{rombach2022high} & 50 & \textbf{5.11} & \textbf{0.49} & 2.90* \\
    LSGM \cite{vahdat2021scorebased} & 23 & 7.22 & - & - \\
    WaveDiff \cite{phung2023wavediff} & 2 & 5.94 & 0.37 & 0.08 \\
    DDGAN \cite{xiao2021tackling} & 2 & 7.64 & 0.36 & - \\
    Score SDE \cite{song2020score} & 4000 & 7.23 & - & - \\
    \midrule
    \multicolumn{5}{c}{\textbf{CelebA-HQ 512}} \\
    \midrule
    Ours (ADM) & 94 & \textbf{6.35} & \textbf{0.41} & 5.13 \\
    WaveDiff \cite{phung2023wavediff} & 2 & 6.40 & 0.35 & 0.59 \\
    DDGAN \cite{xiao2021tackling} & 2 & 8.43 & 0.33 & 1.49 \\
    \bottomrule
\end{tabu}
}
}
\label{tab:celeb256}
}
\subfloat[FFHQ $256\times256$]{
% \raisebox{-.28\height} {
\resizebox{.479\linewidth}{!}{
  \begin{tabu}[t]{l c c c c}
    \toprule
      % \vspace{10mm}
    Model &NFE & FID$\downarrow$ & Recall$\uparrow$ & Time(s)$\downarrow$ \\
    \midrule
    Ours (ADM) & 84 & 8.07 & 0.40 & 5.1 \\
    Ours (DiT L/2) & 88 & \textbf{4.55} & 0.48 & 1.67 \\
    \midrule
    LDM \cite{rombach2022high} & 50 & 4.98 & \textbf{0.50} & 2.90* \\

    ImageBART \cite{esser2021imagebart} & 3 & 9.57 & - & - \\
    \midrule
    ProjectedGAN \cite{Sauer2021NEURIPS} & 1 & 3.08 & 0.46 & - \\
    StyleGAN \cite{karras2019style} & 1 & 4.16 & 0.46 & - \\
    %CIPS & 2.92 & - \\
    \bottomrule
      \vspace{-70mm}
    \end{tabu}
    }
    }
    \label{tab:ffhq}
% }
\vspace{-6mm}
\label{fig:face256}
\end{table}

\setlength{\tabcolsep}{3pt}
\begin{table}[t]
\centering
\caption{Results on LSUN datasets. (*) is measured on our machine }
\vspace{2mm}
\subfloat[LSUN-Church $256\times256$]{
\raisebox{.0\height} {
% \begin{tabu} to 0.4\textwidth{ccc}
  \begin{tabu} to 0.4\textwidth{l c c c}
    \toprule
    Model & FID$\downarrow$ & Recall$\uparrow$ & Time(s)$\downarrow$ \\
    \midrule
    Ours (ADM) & 7.7 & 0.39 & 3.5 \\
    % Ours (DiT L/2) & 6.62 & 0.47 & 1.67 \\
    Ours (DiT L/2) & 5.54 & 0.48 & 1.67 \\
    FM \cite{lipman2023flow}  &10.54 & - &-\\
    \midrule
    LDM \cite{rombach2022high} & 4.02 & 0.52 & 3.79* \\
    
    WaveDiff \cite{phung2023wavediff} & 5.06 & 0.40 & 1.54 \\
    DDPM \cite{ho2020denoising} & 7.89 & - & - \\
    ImageBART \cite{esser2021imagebart} & 7.32 & - & - \\
    %Gotta Go Fast & 25.67 & - \\
    \midrule
    % PGGAN \cite{karras2017progressive} & 6.42 & - & - \\
    StyleGAN \cite{karras2019style} & 4.21 & - & - \\
    StyleGAN2 \cite{karras2020training} & 3.86 &0.36 & - \\
    ProjectedGAN \cite{Sauer2021NEURIPS} & \textbf{1.59} & 0.44 & - \\
    % DiffGAN \cite{wang2023diffusiongan} & 1.85 & \textbf{0.65} & - \\
    
    \bottomrule
\end{tabu}
% \vspace{-10mm}
}
\label{tab:church}
}
~
\subfloat[LSUN-Bedrooms $256\times256$]{
% \raisebox{-.1\height} {
  \begin{tabu} to 0.4\textwidth{l c c c}
    \toprule
    Model & FID$\downarrow$ & Recall$\uparrow$ & Time(s)$\downarrow$ \\
    \midrule
    Ours (ADM) & 7.05 & 0.39 & 5.42 \\
    % Ours (DiT L/2) & 5.62 & 0.41 & 1.67 \\
    Ours (DiT L/2) & 4.92 & 0.44 & 1.67 \\
    
    \midrule
    LDM \cite{rombach2022high} & 2.95 & 0.48 & 2.90* \\
    DDPM \cite{ho2020denoising} & 4.9 & - & - \\
    ImageBART \cite{esser2021imagebart} & 5.51 & - & - \\
    ADM \cite{dhariwal2021diffusion} & \textbf{1.90} & \textbf{0.51} & - \\
    \midrule
    PGGAN \cite{karras2017progressive} &8.34 &- &-\\
    StyleGAN \cite{karras2019style} & 2.35 & 0.48 & - \\
    ProjectedGAN & 1.52 & 0.34 & - \\
    DiffGAN \cite{wang2023diffusiongan} & \textbf{1.43} & \textbf{0.58} & - \\
    
    \bottomrule
% \vspace{2mm}
    \end{tabu}
    % }
    \label{tab:bed}
}
\label{fig:lsun}
\vspace{-6mm}
\end{table}

%%%%%%
\setlength{\tabcolsep}{3pt}
\begin{table}[t]
\centering
\caption{Quantitative results on conditional tasks. }
\vspace{2mm}
\subfloat[Conditional ImageNet ($256$)]{
% \raisebox{.0\height} {
\resizebox{.32\linewidth}{!}{
  \begin{tabu}{l c c c}
    \toprule
    Model & FID$\downarrow$ & Recall$\uparrow$ & Params \\
    \midrule
    Ours (ADM) & 16.71 & 0.56 & 387M \\
    + cfg=1.25 & 8.58 & 0.50 & 387M \\
    % Ours (DiT B/2) & 24.99 & 0.37 & 130M \\
    % + cfg=1.5 & 5.76 & 0.34 & 130M \\
    Ours (DiT B/2) & 20.38 & 0.56 & 130M \\
    + cfg=1.25 & 8.77 & 0.48 & 130M \\
    + cfg=1.5 & \textbf{4.46} & 0.42 & 130M \\
    \midrule
    LDM-8 \cite{rombach2022high} & 15.51 & \textbf{0.63} & 395M \\
    LDM-8-G & 7.76 & 0.35 & 506M \\
    % LDM-4 & 10.56 & 0.62 \\
    % LDM-4-G (cfg=1.5) & 3.60 & 0.48 \\
    \midrule
    DiT-B/2 \cite{Peebles2022DiT} & 43.47 & - & 130M \\
    % DiT-XL/2 & 9.62 & 0.67 \\
    % DiT-XL/2-G (cfg=1.25) & 3.22 & 0.62 \\
    % \midrule
    % ADM \cite{dhariwal2021diffusion} & 10.94 & 0.63 \\
    % ADM-U & 7.49 & 0.63 \\
    % ADM-G & 4.59 & 0.52 \\
    % ADM-G, ADM-U & 3.94 & 0.53 \\
    \bottomrule
\vspace{-42mm}
\end{tabu}
}
% }
\label{tab:imagenet}
}
\subfloat[Semantic-to-image]{
% \raisebox{-.03\height} {
\resizebox{.205\linewidth}{!}{
  \begin{tabu}{l c}
    \toprule
        Model & FID$\downarrow$ \\
        \toprule
        Ours (ADM) & 26.3\\
        \midrule
        Pix2PixHD \cite{wang2018high} &38.5 \\
        SPADE \cite{park2019semantic} &29.2 \\
        DAGAN  \cite{tang2020dual} &29.1 \\
        CLADE \cite{tan2021efficient} &30.6 \\
        SCGAN \cite{wang2021image} &20.8 \\
        SDM \cite{wang2022semantic} & \textbf{18.8} \\
        \bottomrule
% \vspace{0.1mm}
    \end{tabu}
    }
    \label{tab:sem2im}
    }
% }
\subfloat[Image Inpainting]{
% \raisebox{-.03\height} {
\resizebox{.37\linewidth}{!}{
  \begin{tabu}{l c c c}
    \toprule
        Model & FID$\downarrow$ &P-IDS$\uparrow$ &U-IDS$\uparrow$\\
        \toprule
        Ours (ADM) & 4.09 &13.25 &21.59\\
        \midrule
        MAT \cite{li2022mat} &\textbf{2.94} &\textbf{20.88} &\textbf{32.01}\\
        LaMa \cite{suvorov2022resolution} &3.98  &8.82 &22.57\\
        ICT  \cite{wan2021high} &5.24 &4.51 &17.39\\
        MADF \cite{zhu2021image} &10.43 &6.25 &14.62\\
        DeepFill v2 \cite{yu2019free} &5.69  &6.62 &16.82\\
        EdgeConnect \cite{nazeri2019edgeconnect} &5.24  &5.61 &15.65\\
        \bottomrule
% \vspace{12mm}
    \end{tabu}
    }
    \label{tab:inpainting}
    }
% }
\label{fig:conditional}
\vspace{-5mm}
\end{table}

% \begin{table}[t]
%   \centering
    
%   \begin{tabular}{l c c}
%     \toprule
%     Model & FID$\downarrow$ & Recall$\uparrow$ \\
%     \midrule
%     Ours (ADM) &  &  \\
%     Ours (DiT B/2) & 50.73 & \\
%     Ours-G (DiT B/2, cfg=3.) & 9.29 &  \\
%     \midrule
%     LDM-8 \cite{rombach2022high} & 15.51 & 0.63 \\
%     LDM-8-G & 7.76 & 0.35 \\
%     LDM-4 & 10.56 & 0.62 \\
%     LDM-4-G (cfg=1.5) & 3.60 & 0.48 \\
%     \midrule
%     DiT-B/2 \cite{Peebles2022DiT} & 43.47 & - \\
%     DiT-XL/2 & 9.62 & 0.67 \\
%     DiT-XL/2-G (cfg=1.25) & 3.22 & 0.62 \\
%     \midrule
%     ADM \cite{dhariwal2021diffusion} & 10.94 & 0.63 \\
%     ADM-U & 7.49 & 0.63 \\
%     ADM-G & 4.59 & 0.52 \\
%     ADM-G, ADM-U & 3.94 & 0.53 \\
    
%     \bottomrule
%   \end{tabular}
%     \vspace{2mm}
%   \caption{Results on ImageNet $256\times256$.}
%   \label{tab:imnet}
% \end{table}

\begin{table}[t]
  \centering
  \caption{Ablation study on both adaptive and fixed-step ODE solvers on CelebA-HQ (256).}
  \vspace{1mm}
  \begin{tabular}{l c c c}
    \toprule
    Solver & NFE $\downarrow$ & FID$\downarrow$ & Time (s)$\downarrow$ \\
    \midrule
    Adaptive & 89 & \textbf{5.26} & \textbf{1.7} \\
    \midrule
    Euler & 90 & 5.51 & 1.78 \\
    Heun & \textbf{50} &\textbf{5.26} & 1.83 \\
    \bottomrule
  \end{tabular}
  \label{tab:solver}
\vspace{-4mm}
\end{table}

\subsection{Conditional Generation}

\minisection{Conditional ImageNet \cite{Deng2009ImageNetAL}.} Given the class label, our method can generate desired examples without relying on a pretrained classifier. As mentioned earlier, we select the DiT base backbone as the default network for this conditional task due to its comparability and relatively small size, making it well-suited for large-scale datasets like ImageNet. Our proposed classifier-free velocity field enables flexible control over class conditioning, leading to a significant improvement in image generation results. By incorporating this technique, we observe a remarkable reduction in the FID score for both the DiT and ADM variants. As shown in \cref{tab:imagenet}, for the DiT variant, the FID score has decreased from 20.39 to 4.46 ($\text{cfg}=1.5$) while  for the ADM variant, the measure has similarly reduced from 16.71 to 8.58 ($\text{cfg}=1.25$). Despite our network being significantly smaller than LDM, our classifier-free technique demonstrates superior performance gains compared to the larger model. Due to resource constraints, we are unable to scale our approach to the DiT network with a larger size, such as DiT-L/2 and DiT-XL/2, at this time. However, it is worth noting that if we were able to implement this scaling, it could potentially result in a further enlargement of the gap. This possibility could be explored in future development.

\minisection{Inpainting Task.}
In the inpainting task, we must fill the missing region of input images with appropriate content. Following the training technique from \cite{rombach2022high}, given a masked image $\bx_m $ and the corresponding mask $\mathbf{m}$, we pass the masked image through the encoder to get $\bz_m$ and resize the mask $\mathbf{m}$ to $\Bar{\mathbf{m}}$ which has the same shape with $\bz_m$. To condition on the masked image and mask, we concatenate $\bz_t, \bz_m,$ and $\Bar{\mathbf{m}}$ and use it as input to flow matching network $v_\theta (\bz_t, \bc = [\bx_m, \textbf{m}], t) = v_\theta (\text{concat}[\bz_t, \bz_m, \mathbf{\bar{m}}], t)$. Our experiment is conducted on CelebA-HQ 256 \cite{karras2017progressive} and follows the training and evaluation protocol of MAT \cite{li2022mat}. We split data into 24,183 and 2,993 images for training and evaluation. To evaluate the inpainting performance, we utilize the following perceptual metrics FID \cite{heusel2017gans}, P-IDS \cite{zhao2021large}, and U-IDS \cite{zhang2018unreasonable}. \cref{tab:inpainting} shows that our method outperforms most existing inpainting methods and is comparable to LaMa \cite{suvorov2022resolution}. We achieve the FID score of 4.09 using a very simple concatenation trick, and this result is not far from the SOTA number (2.94) produced by MAT \cite{li2022mat}, which is specifically designed for the inpainting task. Our qualitative result is shown in \cref{fig:inpainting}.

\minisection{Semantic-to-image.} For the semantic-to-image task, we have to synthesize photorealistic images given an input semantic layout $\textbf{m}$ with $\mathcal{C}$ classes. Following the layout conditional technique from \cite{rombach2022high}, the one-hot representation of $\textbf{m}$ is firstly interpolated to have the same spatial resolution with $\bz_t$. We then train a conditional network $\omega_\phi$ parameterized by $\phi$ that takes the interpolated one-hot representation of $\textbf{m}$ as input and reduces the number of channels from $\mathcal{C}$ to the number channels of $\bz_t$, $\textbf{m}_c = \omega_\phi(\textbf{m})$ where $\textbf{m}_c$ have the same shape with $\bz_t$. Similar to inpainting task, the flow matching network $v_\theta (\bz_t,  \bc = \textbf{m}, t) = v_\theta (\text{concat}[\bz_t, \mathbf{m}_c], t)$ takes the concatenation of $\textbf{m}_c$ and $\bz_t$ as input. The conditional network $\omega_\phi$ is jointly trained along with the latent flow matching network $v_\theta$. We evaluate our method on CelebA-HQ 256 \cite{karras2017progressive} with 27,000 images for training and 3000 for evaluation FID score \cite{heusel2017gans}. As shown in \cref{fig:semantic}, our model can generate high-quality images, which are close to the ground truth images with regard to semantic layout. Our simple layout-to-image technique achieves an FID score of 26.3 (see \cref{tab:sem2im}) in comparison with other methods such as SCGAN \cite{wang2021image} using saliency guided technique and SDM \cite{wang2022semantic} with special ADA normalization and classifier-free guidance.

\begin{figure}[!ht]
\centering
\subfloat[Semantic Image Synthesis]{
    \includegraphics[width=0.46\linewidth]{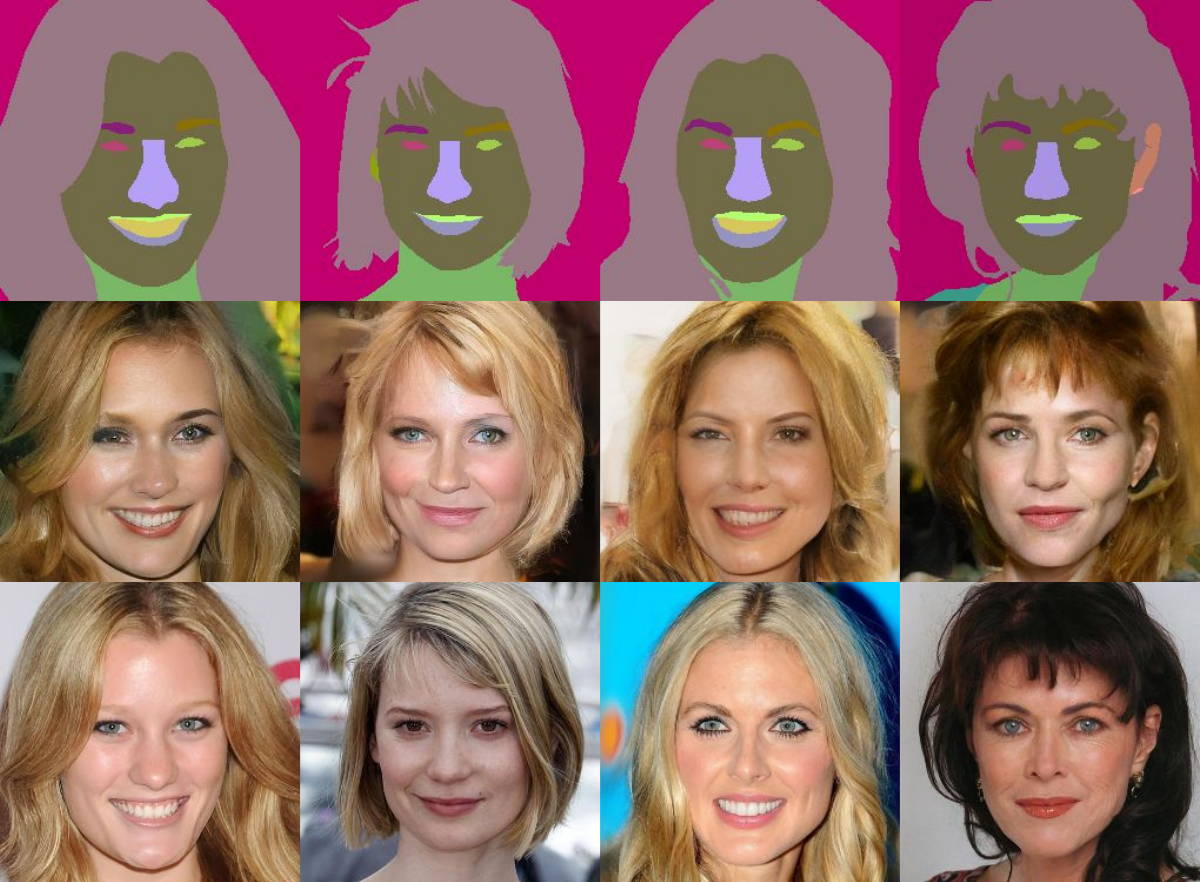}
    \label{fig:semantic}
}
~
\subfloat[Image Inpainting]{
    \includegraphics[width=0.46\linewidth]{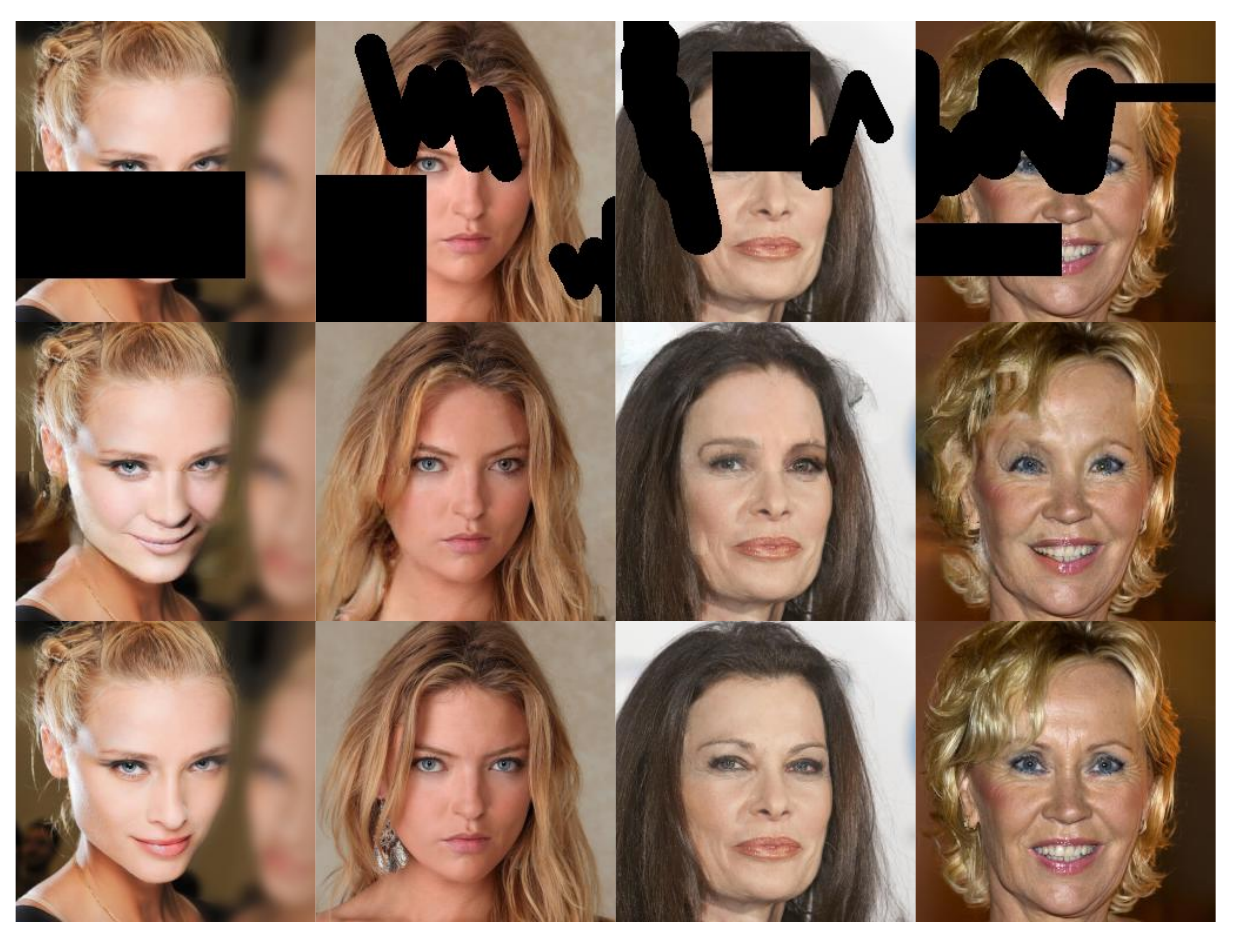}
    \label{fig:inpainting}
}
\caption{\textbf{Image-conditional generation tasks.} For each task, we provide the conditional inputs (top), our results (middle), and the ground truth images (bottom).}
    \label{fig:my_label}
\vspace{-4mm}
\end{figure}

% \begin{table}[!ht]
%     \centering
%     \begin{tabular}{c|c}
%         Model & FID $\downarrow$ \\
%         \toprule
%         Ours (ADM) & 26.3\\
%         \midrule
%         Pix2PixHD \cite{wang2018high} &38.5 \\
%         SPADE \cite{park2019semantic} &29.2 \\
%         DAGAN  \cite{tang2020dual} &29.1 \\
%         CLADE \cite{tan2021efficient} &30.6 \\
%         SCGAN \cite{wang2021image} &20.8 \\
%         SDM \cite{wang2022semantic} & 18.8 \\
%         \bottomrule
%     \end{tabular}
%     \caption{Quantitative comparison with existing methods on semantic image synthesis.}
%     \label{tab:my_label}
% \end{table}
\subsection{Ablation on ODE solvers}
\label{ssec:solvers}
% \begin{tabular}{l c c}
%     \toprule
%     Solver & FID$\downarrow$ & Time (s)$\downarrow$ \\
%     \midrule
%     Ours (ADM) & 7.7 &  &  \\

%     \bottomrule
%   \end{tabular}
%     \vspace{2mm}
%   \caption{Results on LSUN-Church $256\times256$.}
%   \label{tab:solver_comp}
% \end{table}
 As shown in \cref{tab:solver}, the Heun fixed-step solver achieves comparable performance to the adaptive solver with only 50 steps. However, it is slightly slower than the adaptive solver due to the requirement of an additional one-step forward to correct the estimated velocity. In this work, we have selected the adaptive solver instead of the fixed-steps solver, despite acknowledging the promising performance of the Heun solver. Further analysis of fix-steps solvers will be presented in \cref{sec:ode_analysis}.
\label{sec:solvers}

\section{Conclusions}
We introduce a novel approach, called latent flow matching for image generation, which attains competitive performance across various image synthesis tasks and helps bridge the gap of flow matching framework with existing diffusion models. By incorporating a classifier-free guidance for velocity field, we empower the seamless integration of class conditions into the process of flow matching for improving class-conditional image generation. Through extensive empirical evaluations, our method demonstrates the simplicity and effectiveness in enhancing the efficiency of flow matching models. These achievements prove the potential and applicability of our method in driving further generative AI research, particularly for a large-scale system like text-to-image synthesis.

\minisection{Limitation and societal impact.} As our method offers favorable performance across several tasks, it also poses some potential risks related to misinformation dissemination. Training such a model also emits a substantial amount of greenhouse gases.

\bibliographystyle{plain}
\bibliography{biblio}
\newpage
\appendix
\section{Additional background and proofs}
\label{sec:appendix-proof}

Before going into the proof of \Cref{thm:wasserstein-bound}, we introduce additional background on optimal transport.
Throughout this section, without loss of generality, we use $\norm{\cdot}$ to denote the Euclidean norm in $\bbR^d$.

% \subsection{Background on optimal transport}

\begin{definition}[2-Wasserstein distance \cite{villani2009optimal}]
  \label{def:2wass}
  Given probability measures $\alpha, \beta \in \bbR^d$, we define the squared 2-Wasserstein distance as
  \begin{equation}
    \label{eq:2-wasserstein}
    \cW_2^2(\alpha, \beta) \egaldef \min_{\pi \in \Pi(\alpha, \beta)} \int_{\bbR^{d\times d}} \norm{\bx - \by}^2 d\pi(\bx, \by),
  \end{equation}
  where $\Pi(\alpha, \beta)$ is the set of couplings of $\alpha, \beta$, that is, joint distributions on $\bbR^{d\times d}$ whose marginals are $\alpha, \beta$, respectively.
  The distance is finite as long as $\alpha$ and $\beta$ belong to the space $\cP_2(\bbR^d)$ of probability measures over $\bbR^d$ with finite second moments.
\end{definition}
Note that technically, \eqref{eq:2-wasserstein} is known as the Kantorovich formulation of optimal transport.
In general, finding the 2-Wasserstein distance between two arbitrary probability measures involves solving optimization scheme, possibly with entropic regularization \cite{peyre2019computational}.
%
% However, a well-known fact that the Wasserstein distance between two Gaussian distributions admit a closed form, called the Bures-Wasserstein (or Fréchet distance) \cite{dowson1982frechet,takatsu2011wasserstein}.
% %
% \begin{definition}[Bures-Wasserstein distance]
% The 2-Wasserstein distance between two Gaussian distributions $\cN(\ba, \bA)$ and $\cN(\bb, \bB)$ has the closed form formula
% \begin{equation}
%     \cW_2^2(\cN(\ba, \bA), \cN(\bb, \bB)) = \norm{\ba - \bb}^2_2 + \tr{\bA} + \tr{\bB} - 2\tr{\bA^{1/2}\bB\bA^{1/2}}^{1/2}.
% \end{equation}
% %
% \end{definition}

We also need the following result for the proof of our theoretical analysis.
\begin{proposition}[Proposition 3 in \cite{albergo2022building}]
  \label{prop:3-albergo2022}
  Let $p_t$ be the density of $\bx_t$ the interpolation path, defined by the dynamic in \eqref{eq:ode-eulerian}.
  Given a velocity field $\hat{v}(\bx_t, t)$, we define $\hat{p}_t$ as the solution of the initial value follow problem
  \begin{equation}
    \partial_t \hat{p}_t  = -\mathrm{div}(\hat{v}(\bx, t) \hat{p}_t), \quad \hat{p}_0 = p_0.
  \end{equation}
  Assume that $\hat v_t(\bx, t) $ is continuously differentiable in $(\bx)$ and Lipschitz in $\bx$, uniformily on $(\bx, t) \in \bbR^d\times [0,1]$ with Lipschitz constant $\hat L$.
  Then we have the following bound on the square of the $W_2$ distance between $p_1$ and $\hat p_1$
\begin{equation}
    W^2_2(p_1,\hat p_1) \leq  e^{1+2\hat L} \int_0^1 \int_{\bbR^d}  \norm{v_t(\bx_t, t) -  \hat v_t(\bx_t, t)}^2 p_t \mathrm{d}x\mathrm{d}t.
  \end{equation}
\end{proposition}

% Finally, we restate \cref{thm:wasserstein-bound} with a slight modification to the constant of the bound. We remark that this does not change the interpretation of the theorem as stated in the last paragraph of \cref{ssec:theory}, but just rather to compact the constant in our main result. 
%

%
% \subsection{}
\begin{proof}[Proof of \cref{thm:wasserstein-bound}]
  By \cref{def:2wass}, we have
  \begin{align}
    \cW_2^2(p_0, \hat{p}_0) &\egaldef \min_{\pi \in \Pi(\alpha, \beta)} \int_{\bbR^{d\times d}} \norm{\bx_0 - \hat{\bx}_0}^2 d\pi(\bx_0, \hat{\bx}_0) \\
                            &= \min_{\pi \in \Pi(\alpha, \beta)} \int_{\bbR^{d\times d}} \norm{g_{\tau}(\bz_0) - \Delta_{f_{\phi},g_{\tau}}(\bx) - \hat{\bx}_0}^2 d\pi(g_{\tau}(\bz_0) - \Delta_{f_{\phi},g_{\tau}}(\bx), \hat{\bx}_0) \\
                            &= \norm{\Delta_{f_{\phi},g_{\tau}}(\bx)}^2 + \min_{\pi \in \Pi(\bar{\alpha}, \beta)} \int_{\bbR^{d\times d}} \norm{g_{\tau}(\bz_0) - g_{\tau}(\hat{\bz}_0)}^2 d\pi(g_{\tau}(\bz_0), g_{\tau}(\hat{\bz}_0)).
  \end{align}
  where we use the assumption that the stochastic reconstruction $g_{\tau}(\bz_0)$ is a distance $\Delta_{f_{\phi},g_{\tau}}(\bx)$ away to the real data sample $\bx_0$, and $\hat{\bx}_0 \egaldef  g_{\tau}(\hat{\bz}_0)$.
  This implies
  \begin{equation}
      \cW_2^2(p_0, \hat{p}_0) \leq \norm{\Delta_{f_{\phi},g_{\tau}}(\bx)}^2 + L^2_{g_{\tau}} \cW_2^2(q^{\phi}_0, \hat{q}^{\phi}_0) %\left[\min_{\pi} \int_{\bbR^{d/h\times d/h}} \norm{\bz_0 - \hat{\bz}_0}^2 d\pi(\bz_0, \hat{\bz}_0) \right]
  \end{equation}
  by Lipschitz continuity of $g_{\tau}$.
  Finally, using \cref{prop:3-albergo2022} to bound $\cW_2^2(q^{\phi}_0, \hat{q}^{\phi}_0)$, we have
  \begin{equation}
    \cW_2^2(p_0, \hat{p}_0) \leq \norm{\Delta_{f_{\phi},g_{\tau}}(\bx)}^2 + L^2_{g_{\tau}} e^{1+2\hat L} \int_0^1 \int_{\bbR^{d/h}}  \norm{v_t(\bz_t, t) -  \hat v_t(\bz_t, t)}^2 q^{\phi}_t \mathrm{d}z\mathrm{d}t.
  \end{equation}
  with a small remark that although in \cite{albergo2022building}, the author use $p_0$ to denote the base (noise), and $p_1$ target (data) distributions, this does not prevent the direct usage of their result in our proof.

\end{proof}

\section{Analysis of fixed-steps ODE solvers}
\label{sec:ode_analysis}
\begin{figure}[t]
\centering
\subfloat[FID vs. Steps]{
    \includegraphics[width=0.46\linewidth]{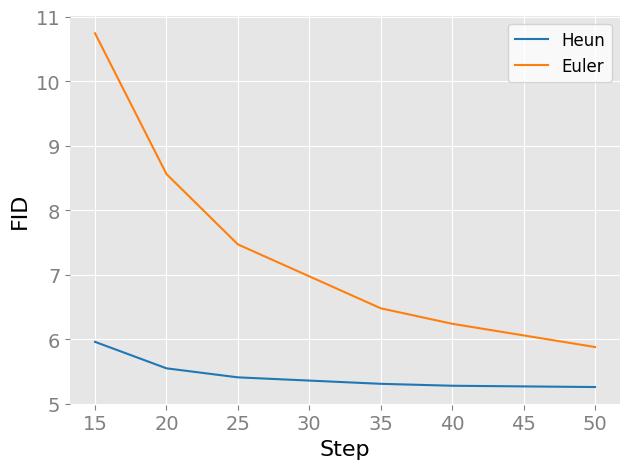}
    \label{fig:solver_fid}
}
~
\subfloat[Time vs. Steps]{
    \includegraphics[width=0.46\linewidth]{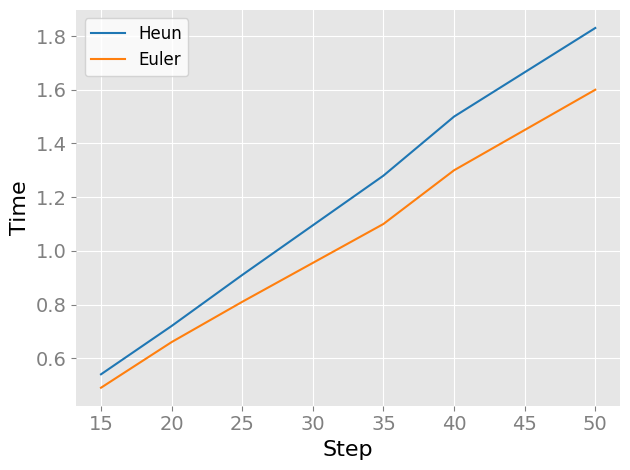}
    \label{fig:solver_time}
}
\caption{Fixed-step ODE solvers on CelebA-HQ 256.}
    \label{fig:solver_analysis}
\vspace{-2mm}
\end{figure}

We perform fixed-step ODE solvers with various numbers of steps, then measure FID and sampling time on each setting. As shown in \cref{fig:solver_analysis}, the Heun solver consistently obtains better FID scores across different settings, but its sampling time is slightly higher than that of the Euler algorithm.

\section{Algorithms}
For unconditional image generation, Training and sampling procedure are included in \cref{alg:utrain} and \cref{alg:usample}. 

Compared to the unconditional model, the training procedure for class-conditional image generation is basically similar but includes an extra class condition in the estimation function $v_\theta(\bz, \bc, t)$ in \cref{alg:ctrain}. The sampling process with the proposed classifier-free 
velocity field is also portrayed in \cref{alg:csample}.

\begin{minipage}[t]{0.48\textwidth}
    \begin{algorithm}[H]
\caption{Unconditional training}\label{alg:utrain}
\kwInput{data $p_0$, VAE encoder $\cE$, velocity estimator $v_\theta$, learning rate $\eta $}\;
\Repeat{convergence}{
    $\bx_0 \sim p_0$, $\bz_0 \gets \mathcal{E}(\bx_0)$\; 
    $\bz_1 \gets \cN(0, \bI)$, $t \sim \cU[0, 1]$\;
    $\bz_t \gets (1 - t)\bz_0 + t\bz_1$\;
    $l \gets \norm{\bz_1 - \bz_0 - v_\theta\left(\bz_t, t\right) }^2_2$\;
    $\theta \gets \theta-\eta \nabla_{\theta} l$\;
}
\end{algorithm}
\end{minipage}
\hfill
\begin{minipage}[t]{0.5\textwidth}
    \begin{algorithm}[H]
\caption{Unconditional Euler sampling}\label{alg:usample}
\kwInput{initial noise $\bz_1$, VAE decoder $\cD$, velocity estimator $v_\theta$, steps $N$}\;
$t_n \gets \frac{n}{N}$\;
\For{$n = N-1$ \KwTo $0$}{
    $z_{t_n} \gets z_{t_{n+1}} + (t_n - t_{n+1}) v_\theta(\bz_{t_{n+1}}, t_{n+1})$
}
$\bx_0 \gets \cD(\bz_0)$\;
\kwOutput{$\bx_0$}\;
\end{algorithm}
\end{minipage}

\begin{minipage}[t]{0.48\textwidth}
    \begin{algorithm}[H]
\caption{Class-conditional training}\label{alg:ctrain}
\kwInput{data $p_0$, VAE encoder $\cE$, velocity estimator $v_\theta$, learning rate $\eta $, probability of unconditional training $p_u$}\;
\Repeat{convergence}{
    $\bx_0, \bc \sim p_0$, $\bz_0 \gets \mathcal{E}(\bx_0)$\;
    $\bc \gets \emptyset$ with probability $p_u$\;
    $\bz_1 \gets \cN(0, \bI)$, $t \sim \cU[0, 1]$\;
    $\bz_t \gets (1 - t)\bz_0 + t\bz_1$\;
    $l \gets \norm{\bz_1 - \bz_0 - v_\theta\left(\bz_t, \bc, t\right) }^2_2$\;
    $\theta \gets \theta-\eta \nabla_{\theta} l$\;
}
\end{algorithm}
\end{minipage}
\hfill
\begin{minipage}[t]{0.5\textwidth}
    \begin{algorithm}[H]
\caption{Class-conditional Euler sampling}\label{alg:csample}
\kwInput{initial noise $\bz_1$, VAE decoder $\cD$, velocity estimator $v_\theta$, steps $N$, class $c$, guidance $\gamma$}\;
$t_n \gets \frac{n}{N}$\;
\For{$n = N-1$ \KwTo $0$}{
    $v_{c} \gets v_\theta(\bz_{t_{n+1}}, c, t_{n+1})$\;
    $v_{u} \gets v_\theta(\bz_{t_{n+1}}, \emptyset, t_{n+1})$\;
    $\tilde{v}_c \gets v_{u} + \gamma(v_{c} - v_{u})$
    $z_{t_n} \gets z_{t_{n+1}} + (t_n - t_{n+1})\tilde{v}_c$
}
$\bx_0 \gets \cD(\bz_0)$\;
\kwOutput{$\bx_0$}\;
\end{algorithm}
\end{minipage}

\section{Additional experiments}

\subsection{Class-conditional image generation}
More comparisons of existing methods on ImageNet $256 \times 256$ are given in \cref{tab:full_imagenet}. It is important to emphasize that our approach, utilizing the DiT variant and the proposed classifier-free velocity guidance, has consistently demonstrated superior performance compared to both the original LDM and DiT-B/2 alternatives, despite having smaller model sizes. Additionally, our method exhibits a notable level of competitiveness when compared to established methods in the field.
\begin{table}[t]
    \centering
    \begin{tabular}{l c c c}
        \toprule
    Model & FID$\downarrow$ & Recall$\uparrow$ & Params \\
    \midrule
    Ours (ADM) & 16.71 & 0.56 & 387M \\
    + cfg=1.25 & 8.58 & 0.50 & 387M \\
    % Ours (DiT B/2) & 24.99 & 0.37 & 130M \\
    % + cfg=1.5 & 5.76 & 0.34 & 130M \\
    Ours (DiT B/2) & 20.38 & 0.56 & 130M \\
    + cfg=1.25 & 8.77 & 0.48 & 130M \\
    + cfg=1.5 & 4.46 & 0.42 & 130M \\
    \midrule
    LDM-8 \cite{rombach2022high} & 15.51 & 0.63 & 395M \\
    LDM-8-G & 7.76 & 0.35 & 506M \\
    LDM-4 & 10.56 & 0.62 & 400M \\
    LDM-4-G (cfg=1.5) & 3.60 & 0.48 & 400M \\
    \midrule
    DiT-B/2 \cite{Peebles2022DiT} & 43.47 & - & 130M \\
    DiT-XL/2 & 9.62 & 0.67 & 675M \\
    DiT-XL/2-G (cfg=1.25) & 3.22 & 0.62 & 675M \\
    \midrule
    ADM \cite{dhariwal2021diffusion} & 10.94 & 0.63 & 554M \\
    ADM-U & 7.49 & 0.63 & 554M \\
    ADM-G & 4.59 & 0.52 & 608M \\
    ADM-G, ADM-U & 3.94 & 0.53 & - \\
    \midrule
    CDM \cite{ho2022cdm} & 4.88 & - & - \\
    \midrule
    ImageBART \cite{esser2021imagebart} & 7.44 & - & 3.5B \\
    \midrule
    VQGAN+T \cite{esser2021taming} & 5.88 & - & 1.3B \\
    MaskGIT \cite{chang2022maskgit} & 6.18 & 0.51 & 227M \\
    RQ-Transformer \cite{lee2022autoregressive} & 3.80 & - & 3.8B \\
    \midrule
    BigGan-deep \cite{brock2018large} & 6.95 & 0.28 & 160M \\
    StyleGAN-XL \cite{sauer2022stylegan} & 2.30 & 0.53 & 166M \\
    \bottomrule
    \end{tabular}
    \vspace{2mm}
    \caption{Class-conditional image generation on ImageNet $256 \times 256$}
    \label{tab:full_imagenet}
\end{table}

\begin{figure}[t]
    \centering
    \includegraphics[width=\linewidth]{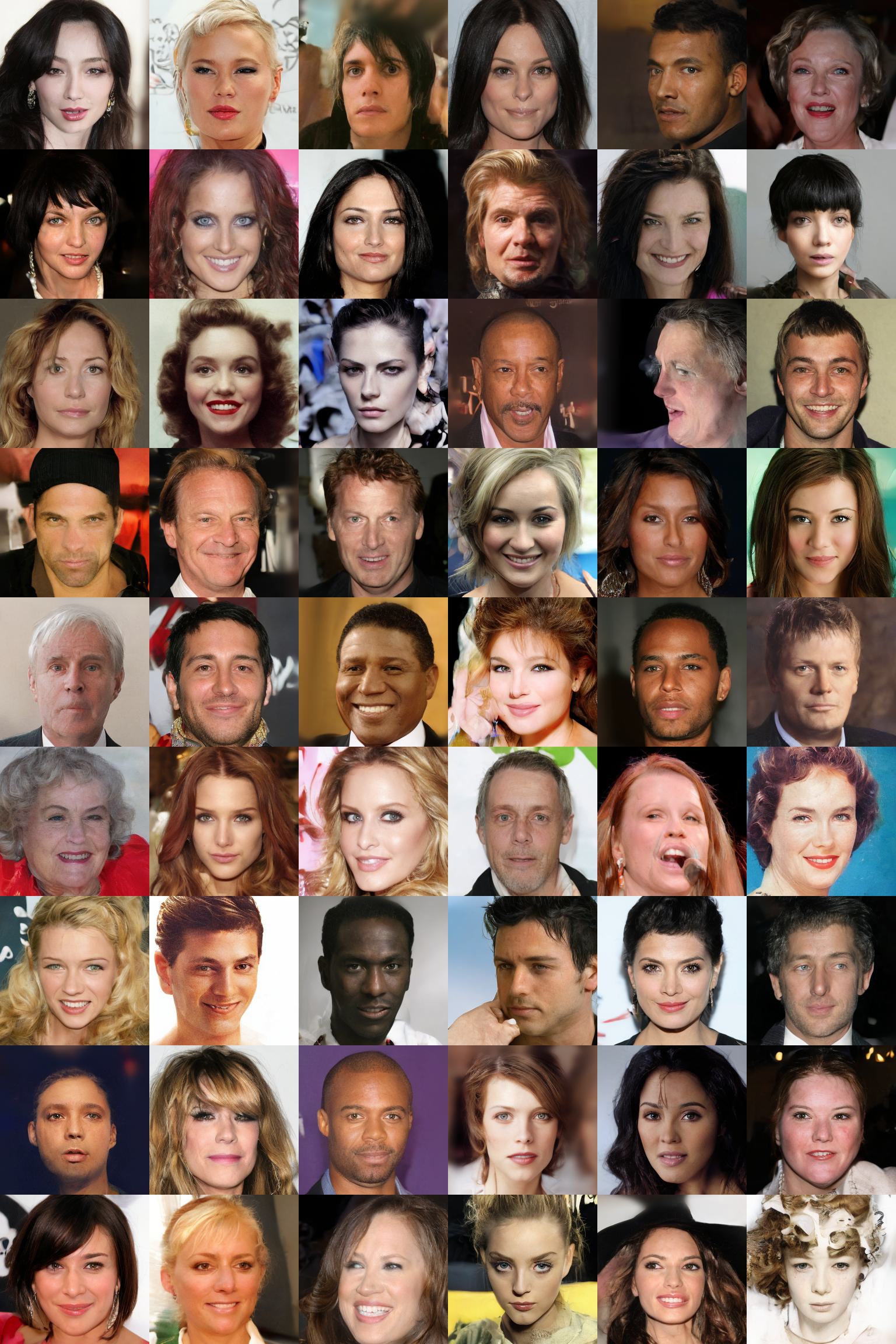}
    % \vspace{-1mm}
    \caption{Non-curated generated samples of CelebA-HQ 256.}
    \label{fig:more_celeb_samples}
\end{figure}

\begin{figure}[t]
    \centering
    \includegraphics[width=\linewidth]{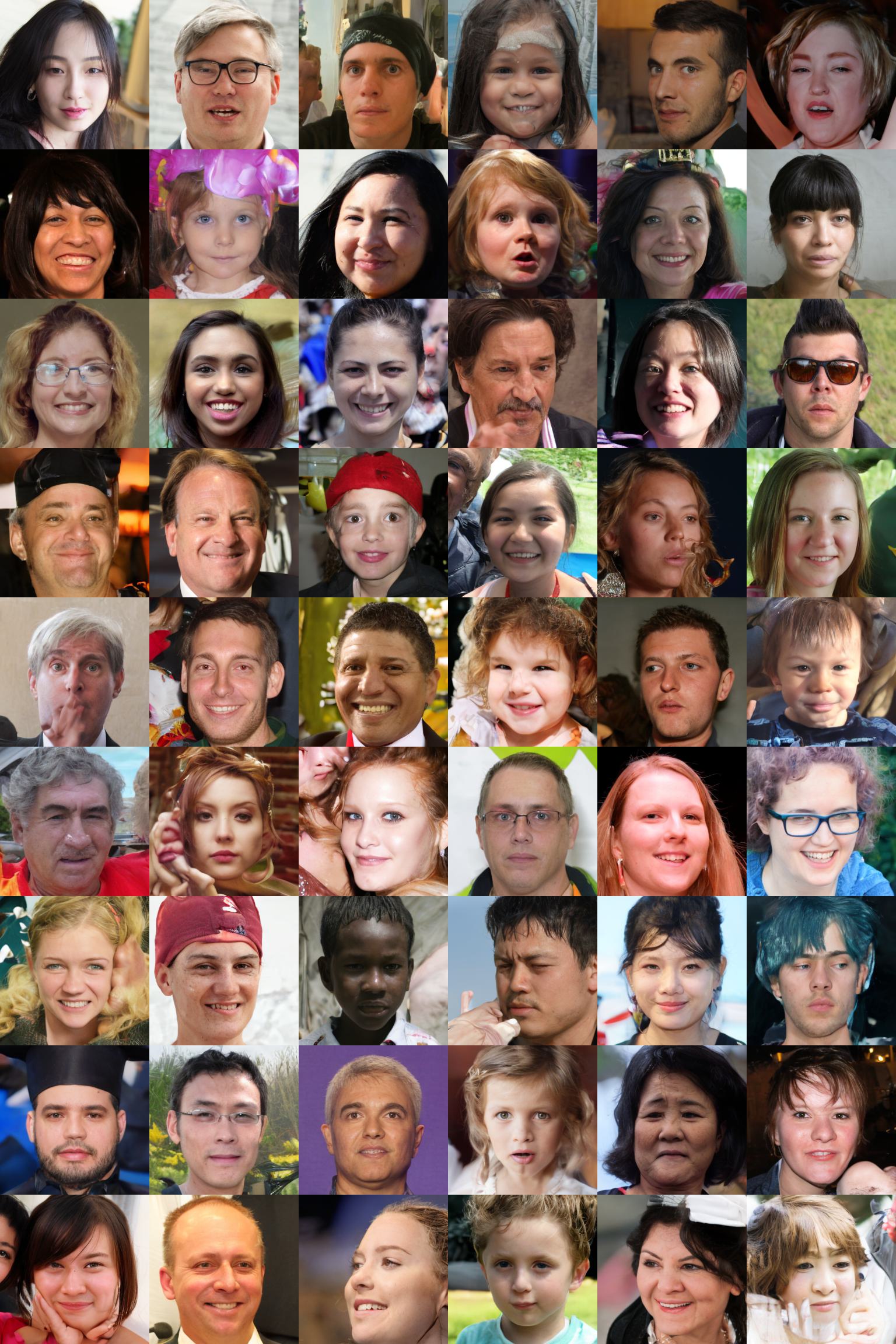}
    % \vspace{-1mm}
    \caption{Non-curated generated samples of FFHQ 256.}
    \label{fig:ffhq_samples}
\end{figure}

\begin{figure}[t]
    \centering
    \includegraphics[width=\linewidth]{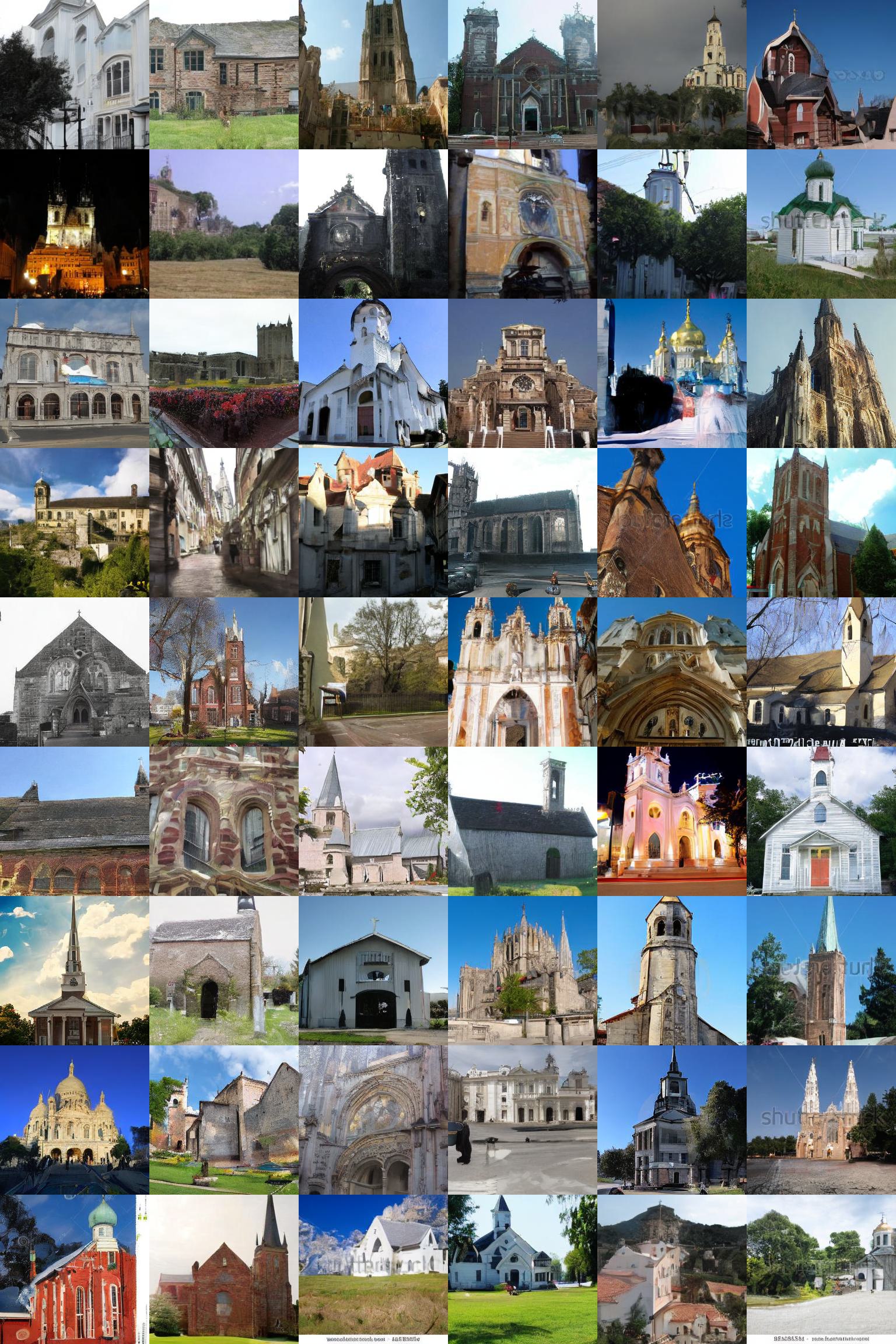}
    % \vspace{-1mm}
    \caption{Non-curated generated samples of LSUN-Church.}
    \label{fig:church_samples}
\end{figure}

\begin{figure}[t]
    \centering
    \includegraphics[width=\linewidth]{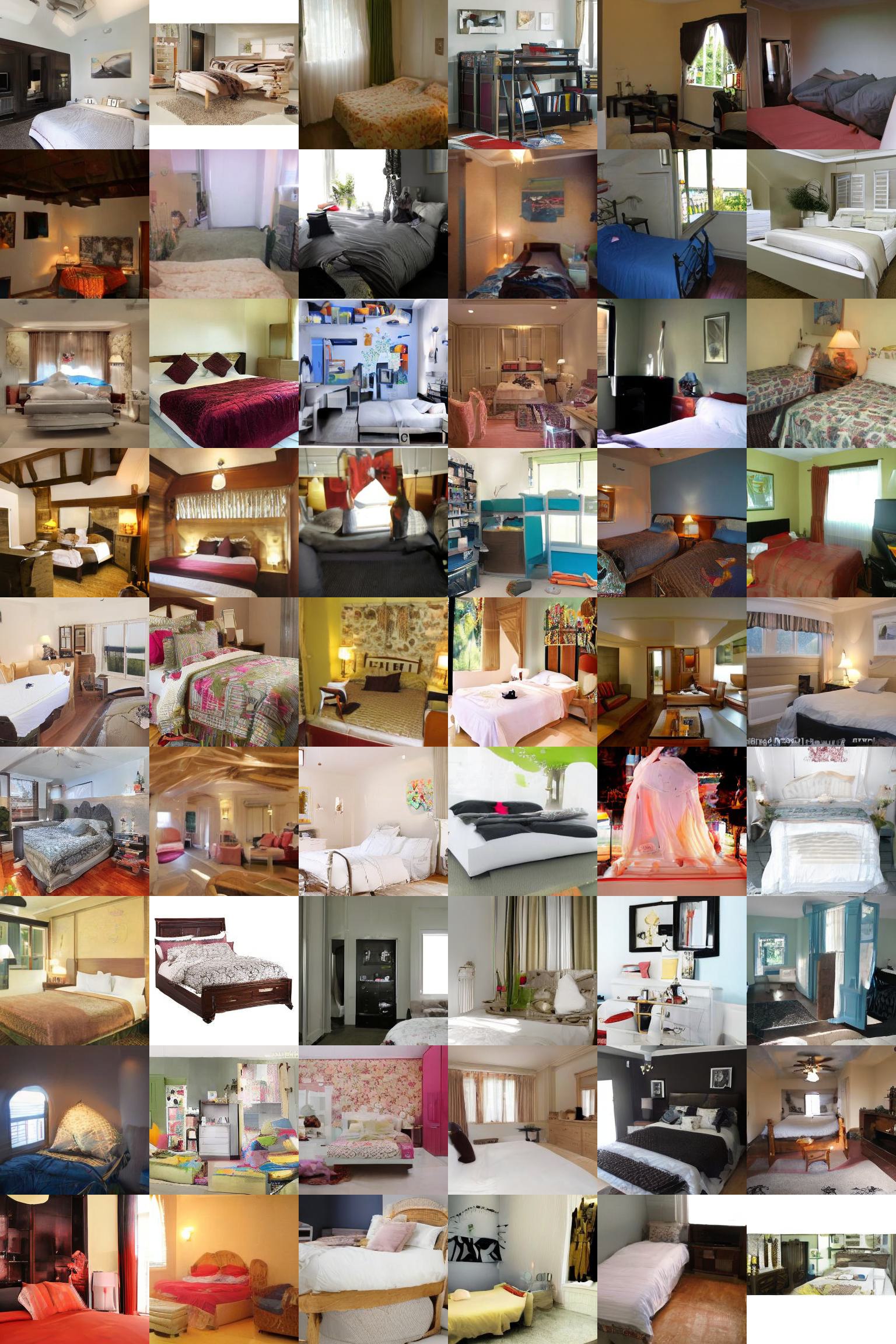}
    % \vspace{-1mm}
    \caption{Non-curated generated samples of LSUN-Bedroom.}
    \label{fig:bed_samples}
\end{figure}

\begin{figure}[t]
    \centering
    \includegraphics[width=\linewidth]{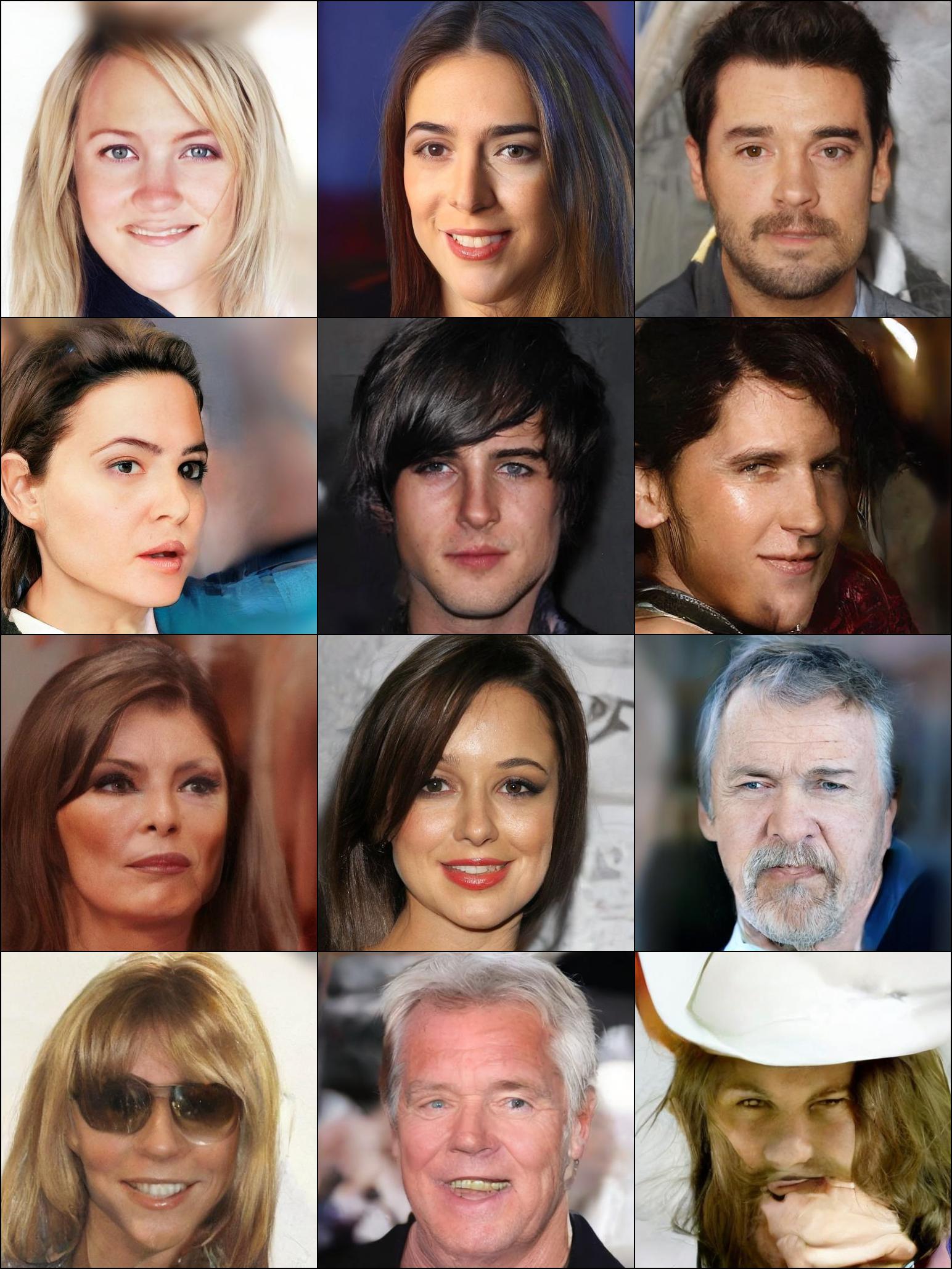}
    % \vspace{-1mm}
    \caption{Non-curated generated samples of CelebA-HQ 512.}
    \label{fig:celeb512_samples}
\end{figure}

% \begin{figure}[t]
%     \centering
%     \includegraphics[width=\linewidth]{figures/supp/samples.jpg}
%     % \vspace{-1mm}
%     \caption{Non-curated generated samples of guided ImageNet DiT B/2 ($\text{cfg}=1.5$)}
%     \label{fig:imnet_samples}
% \end{figure}

\begin{figure}[t]
\centering
\subfloat[Class: 88 - 'Macaw`]{
    \includegraphics[width=0.48\linewidth]{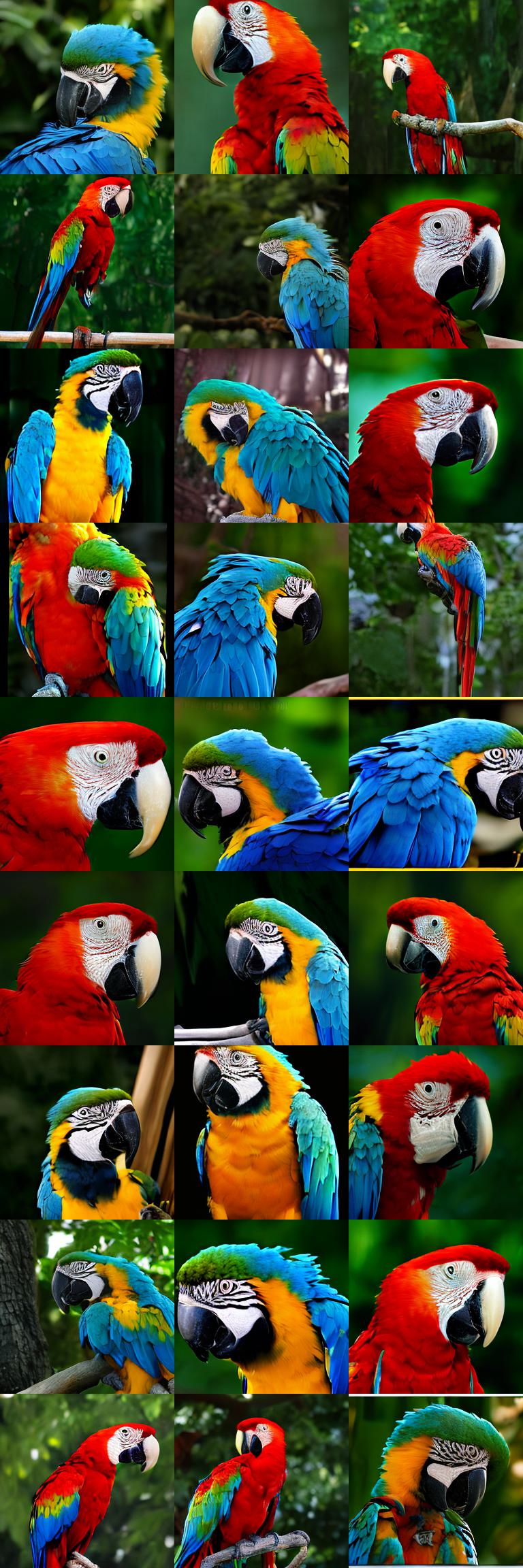}
}
~
\subfloat[Class: 417 - 'Balloon`]{
    \includegraphics[width=0.48\linewidth]{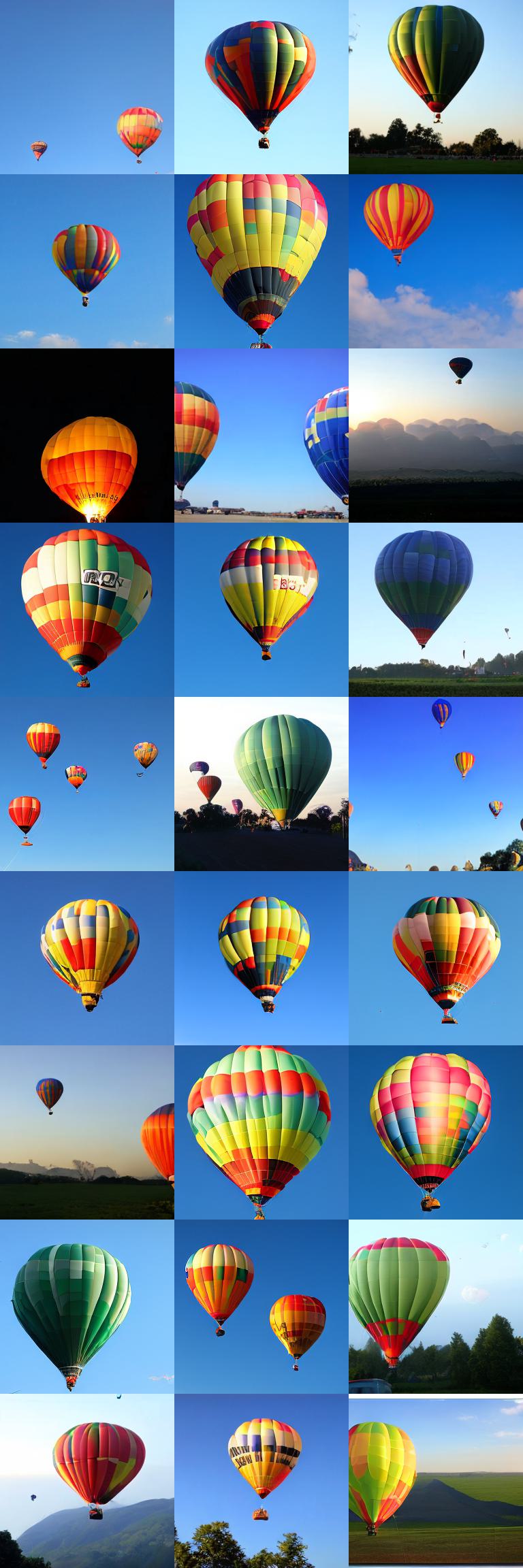}
}

\caption{Non-curated samples of our DiT-B/2 variant on ImageNet (Guidance scale: 4.0).}
    \label{fig:imnet_cfg4.0_00}
\vspace{-2mm}
\end{figure}

\begin{figure}[t]
\centering
\subfloat[Class: 327 - ‘Starfish’]{
    \includegraphics[width=0.48\linewidth]{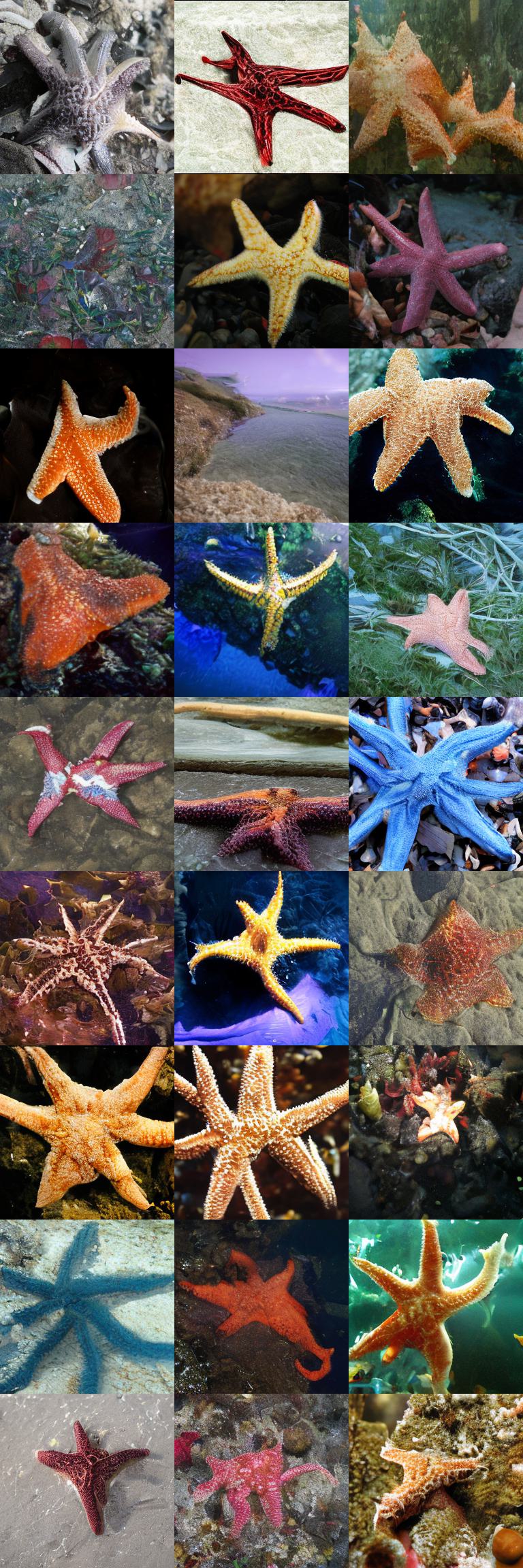}
}
~
\subfloat[Class: 947 - `Mushroom']{
    \includegraphics[width=0.48\linewidth]{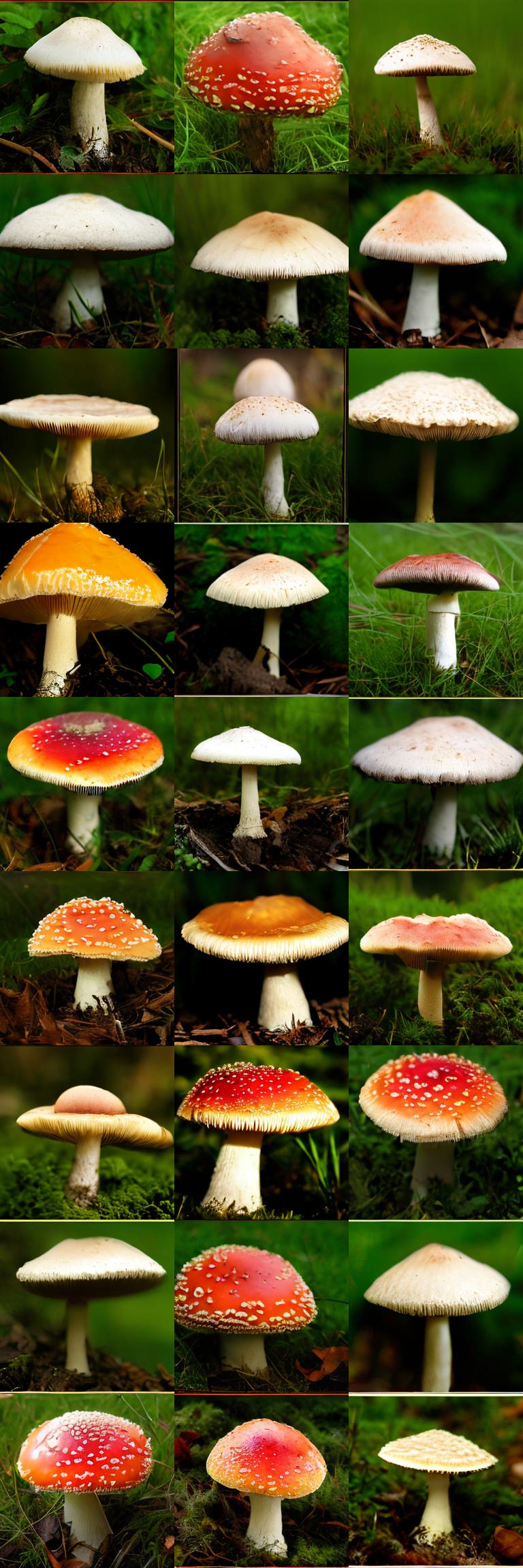}
}
\caption{Non-curated samples of our DiT-B/2 variant on ImageNet (Guidance scale: 4.0).}
    \label{fig:imnet_cfg4.0_01}
\vspace{-2mm}
\end{figure}

\begin{figure}[t]
\centering
\subfloat[Class: 396 - `Lionfish']{
    \includegraphics[width=0.48\linewidth]{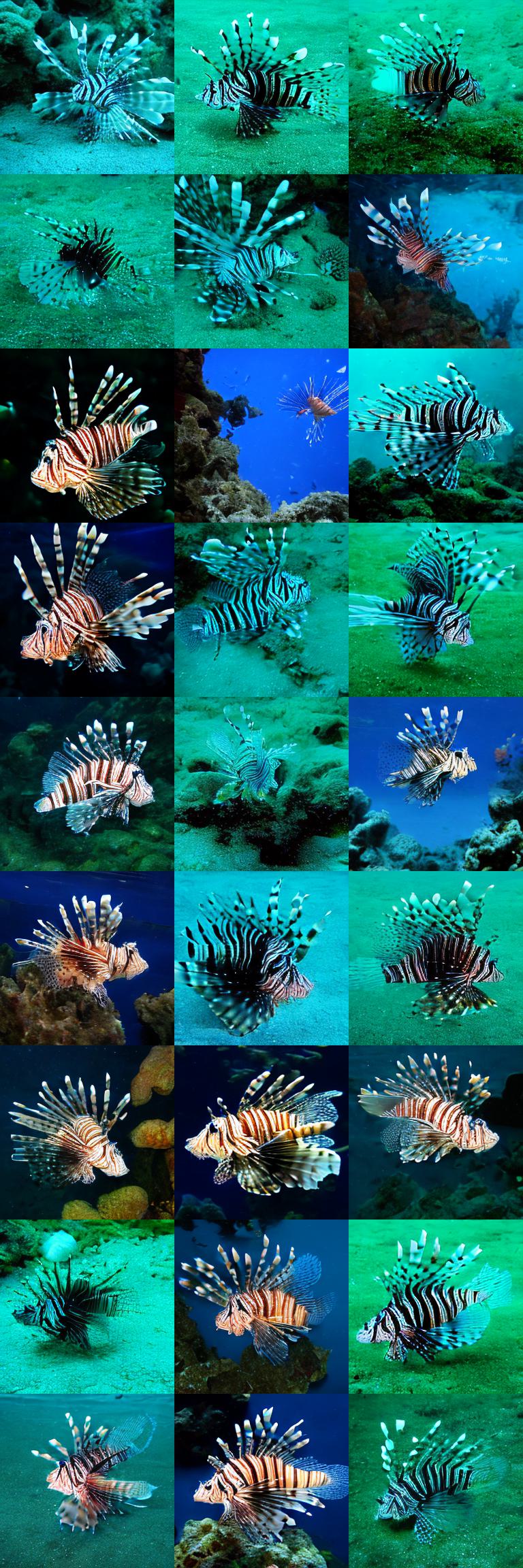}
}
~
\subfloat[Class: 33 - `Loggerhead Turtle']{
    \includegraphics[width=0.48\linewidth]{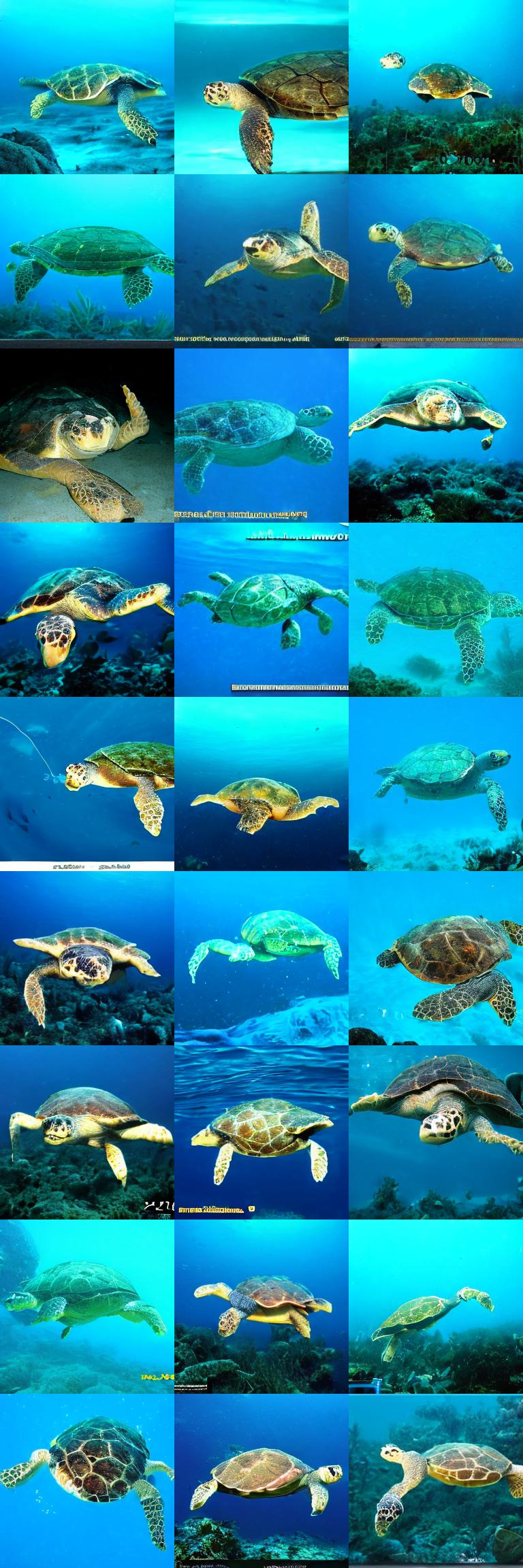}
}
\caption{Non-curated samples of our DiT-B/2 variant on ImageNet ($\text{cfg}=4.0$).}
    \label{fig:imnet_cfg4.0_02}
\vspace{-2mm}
\end{figure}

\begin{figure}[t]
\centering
\subfloat[Class: 449 - ‘Boathouse’]{
    \includegraphics[width=0.48\linewidth]{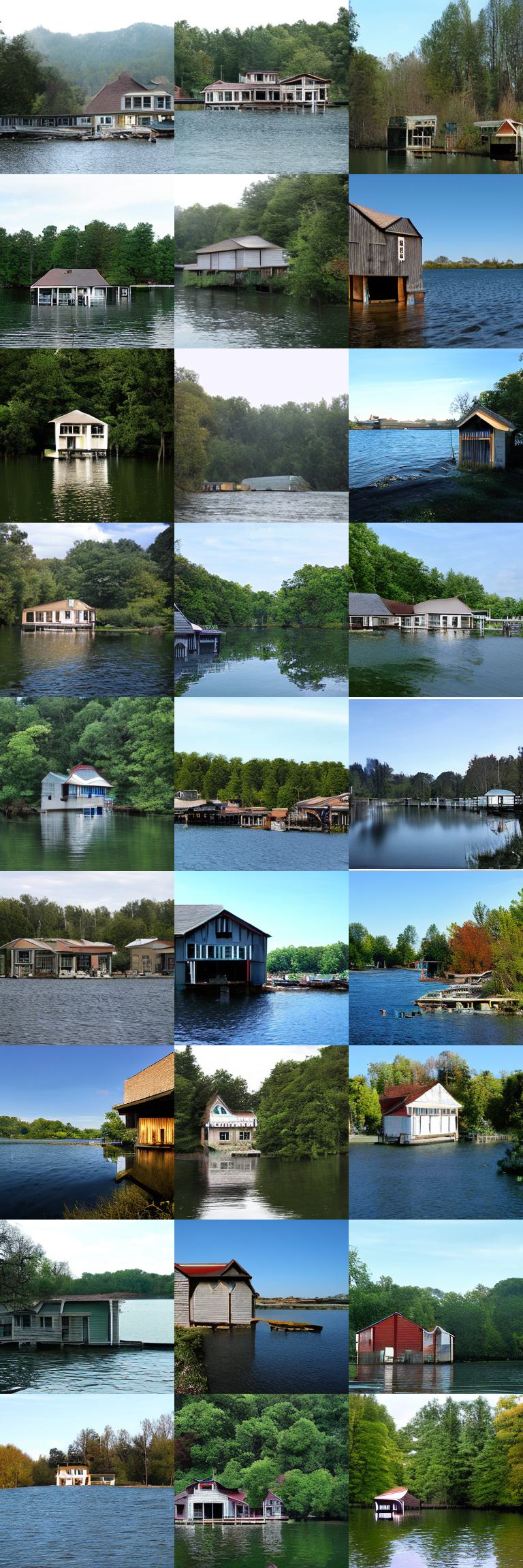}
}
~
\subfloat[Class: 972 - ‘Cliff’]{
    \includegraphics[width=0.48\linewidth]{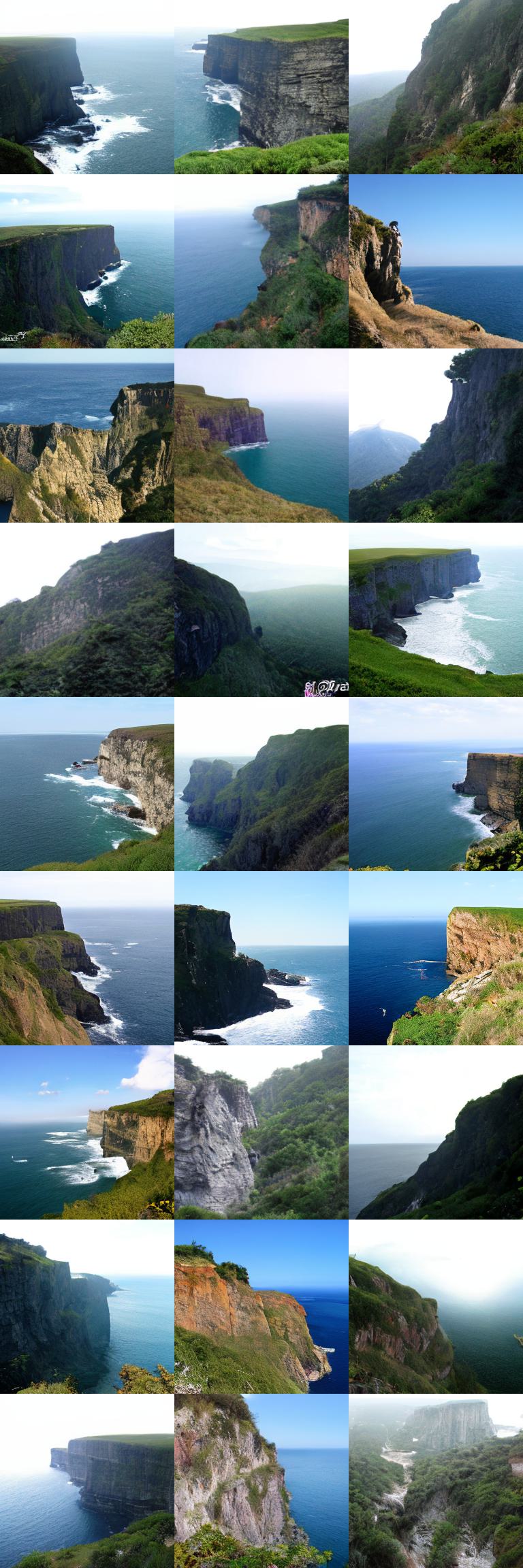}
}

\caption{Non-curated samples of our DiT-B/2 variant on ImageNet ($\text{cfg}=4.0$).}
    \label{fig:imnet_cfg4.0_03}
\vspace{-2mm}
\end{figure}

\begin{figure}[t]
\centering
\subfloat[Class: 207 - `Golden retriever']{
    \includegraphics[width=0.48\linewidth]{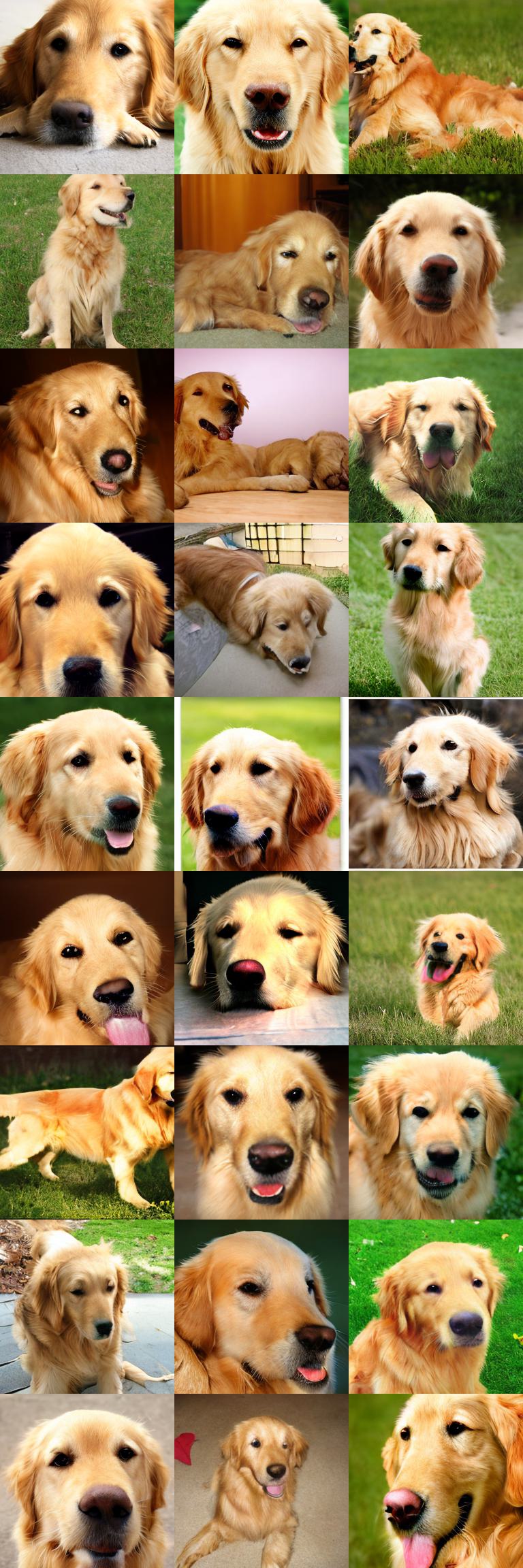}
}
~
\subfloat[Class: 279 - `Arctic fox']{
    \includegraphics[width=0.48\linewidth]{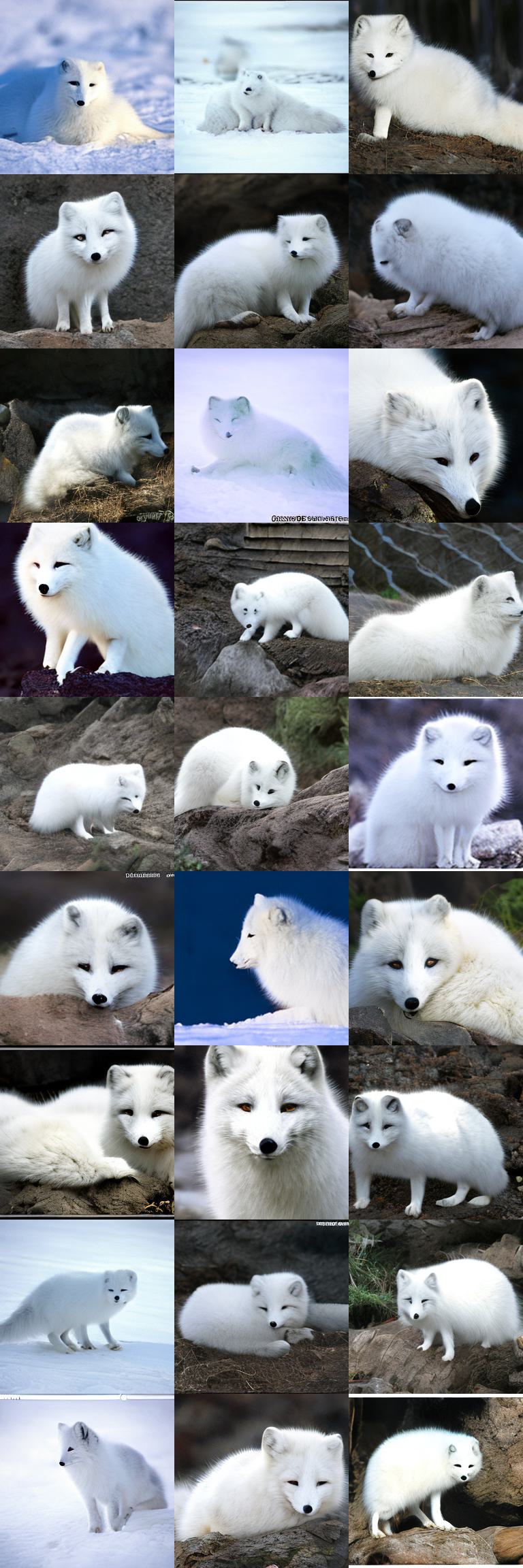}
}

\caption{Non-curated samples of our DiT-B/2 variant on ImageNet ($\text{cfg}=4.0$).}
    \label{fig:imnet_cfg4.0_04}
\vspace{-2mm}
\end{figure}

\begin{figure}[t]
\centering
\subfloat[Class: 89 - `Sulphur-crested cockatoo']{
    \includegraphics[width=0.48\linewidth]{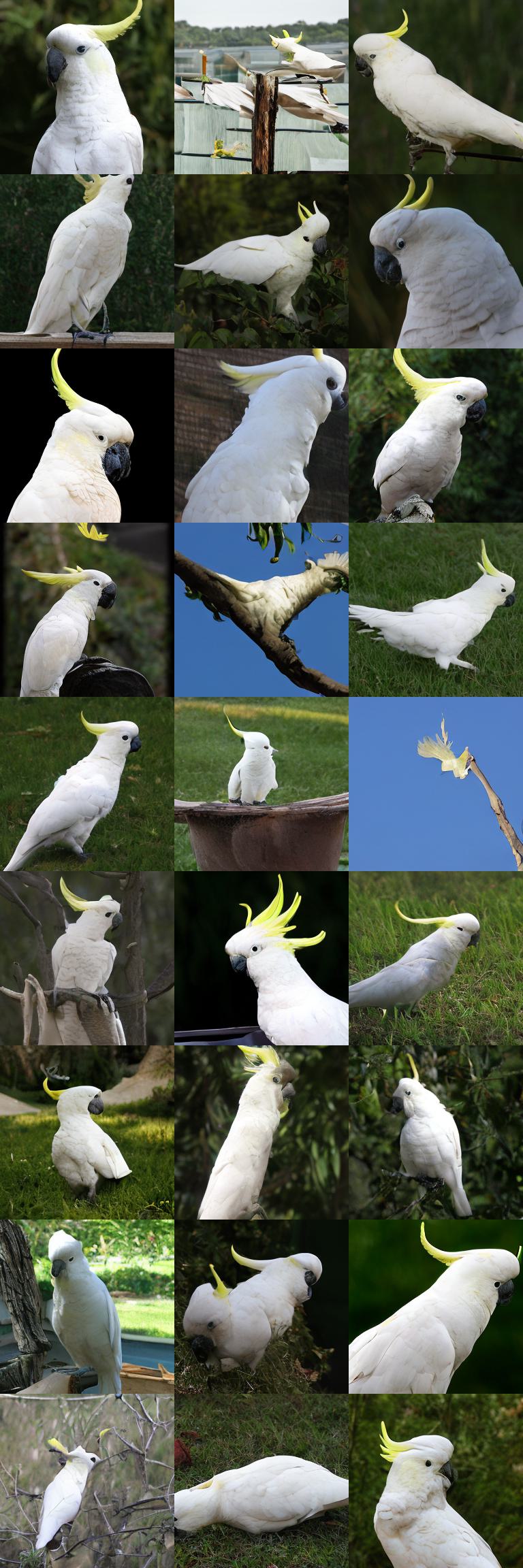}
}
~
\subfloat[Class: 980 - `Volcano']{
    \includegraphics[width=0.48\linewidth]{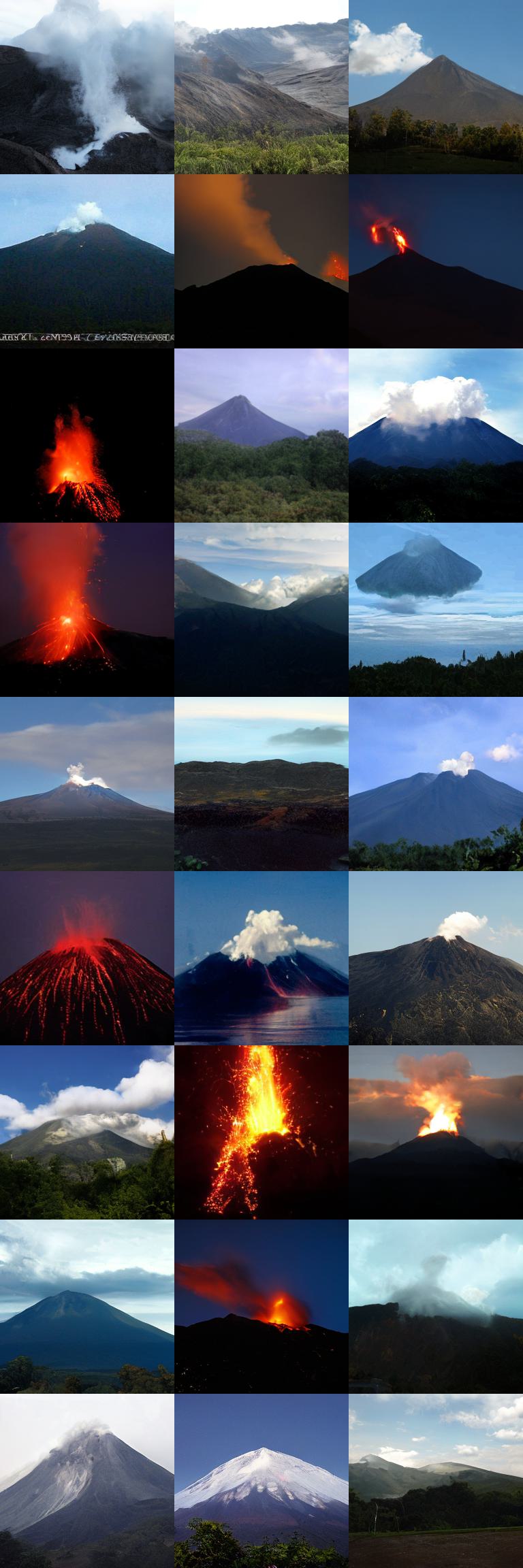}
}
\caption{Non-curated samples of our DiT-B/2 variant on ImageNet ($ \text{cfg}=2.0$).}
    \label{fig:imnet_cfg2.0a}
\vspace{-2mm}
\end{figure}

\begin{figure}[t]
\centering
\subfloat[Class: 973 - `Coral reef']{
    \includegraphics[width=0.48\linewidth]{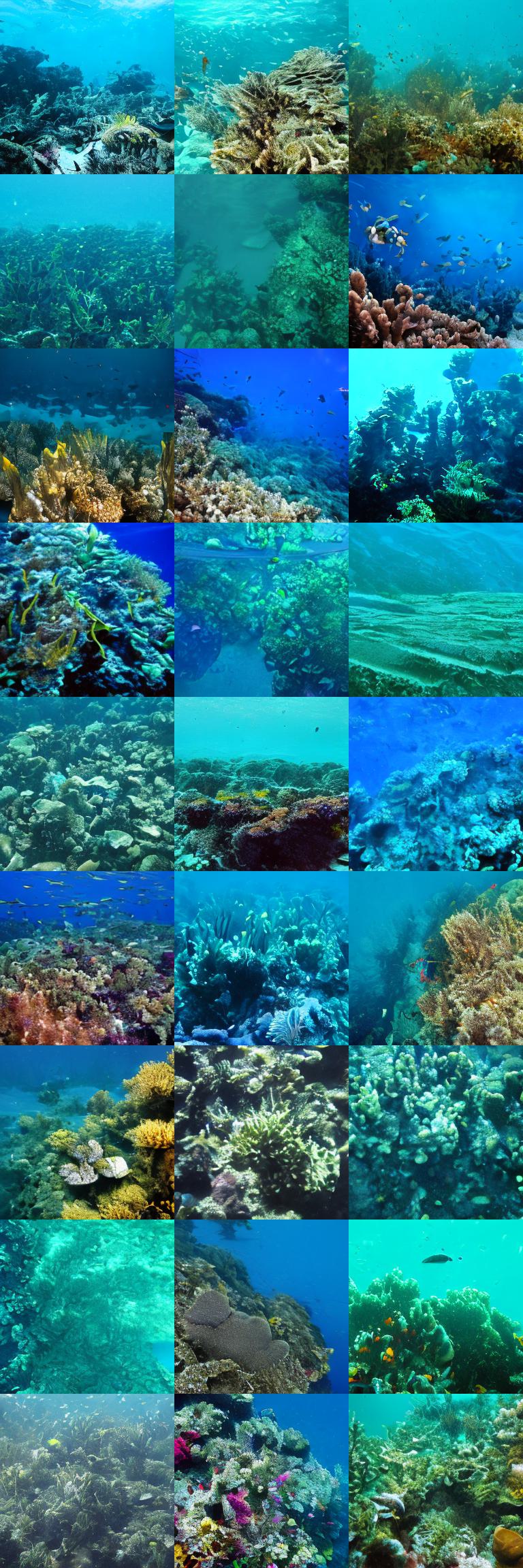}
}
~
\subfloat[Class: 360 - `Otter']{
    \includegraphics[width=0.48\linewidth]{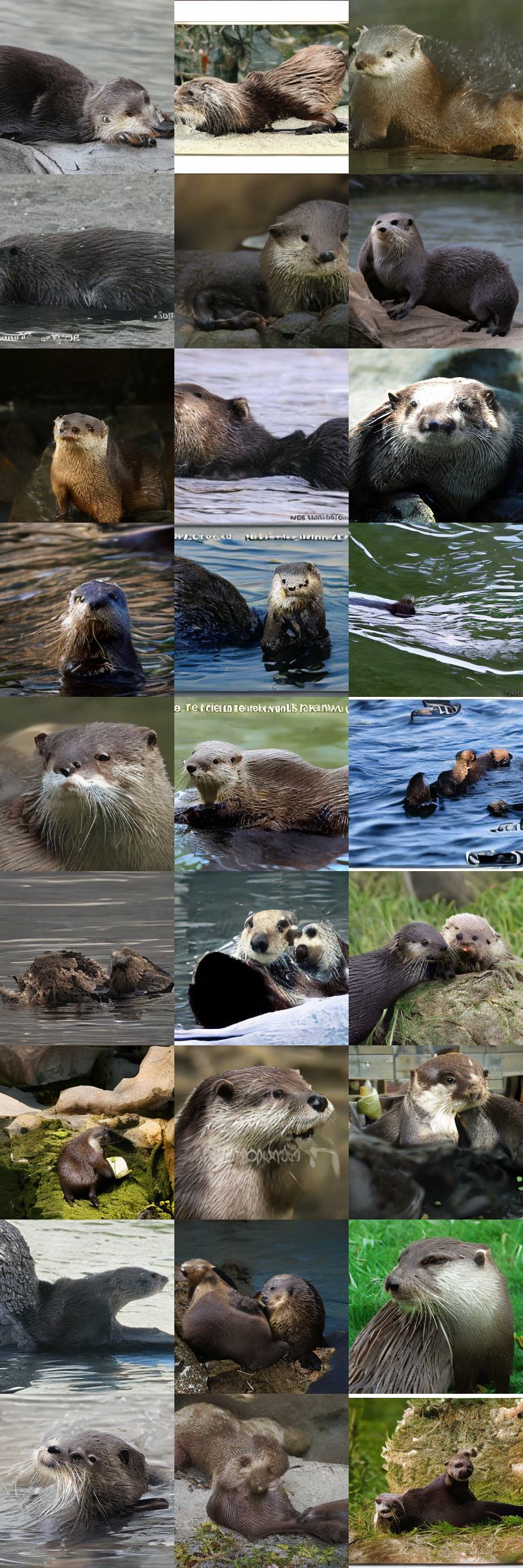}
}
\caption{Non-curated samples of our DiT-B/2 variant on ImageNet ($ \text{cfg}=2.0$).}
    \label{fig:imnet_cfg2.0b}
\vspace{-2mm}
\end{figure}

\begin{figure}[t]
\centering
\subfloat[Class: 937 - `Broccoli']{
    \includegraphics[width=0.48\linewidth]{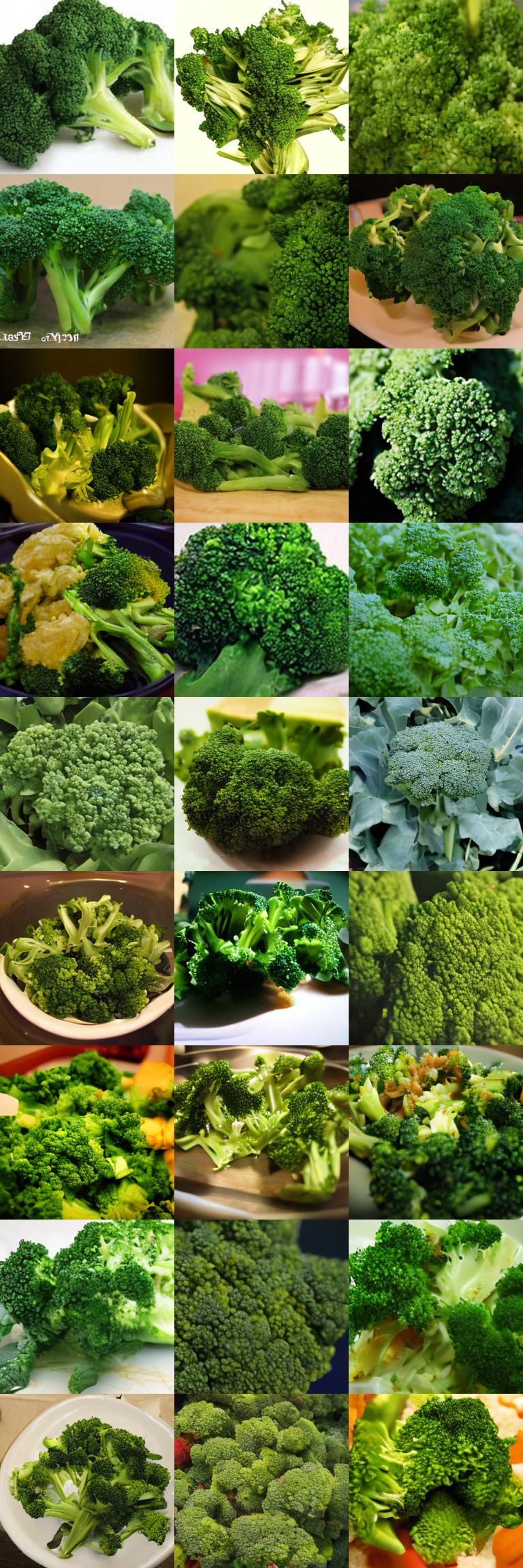}
}
~
\subfloat[Class: 963 - `Pizza']{
    \includegraphics[width=0.48\linewidth]{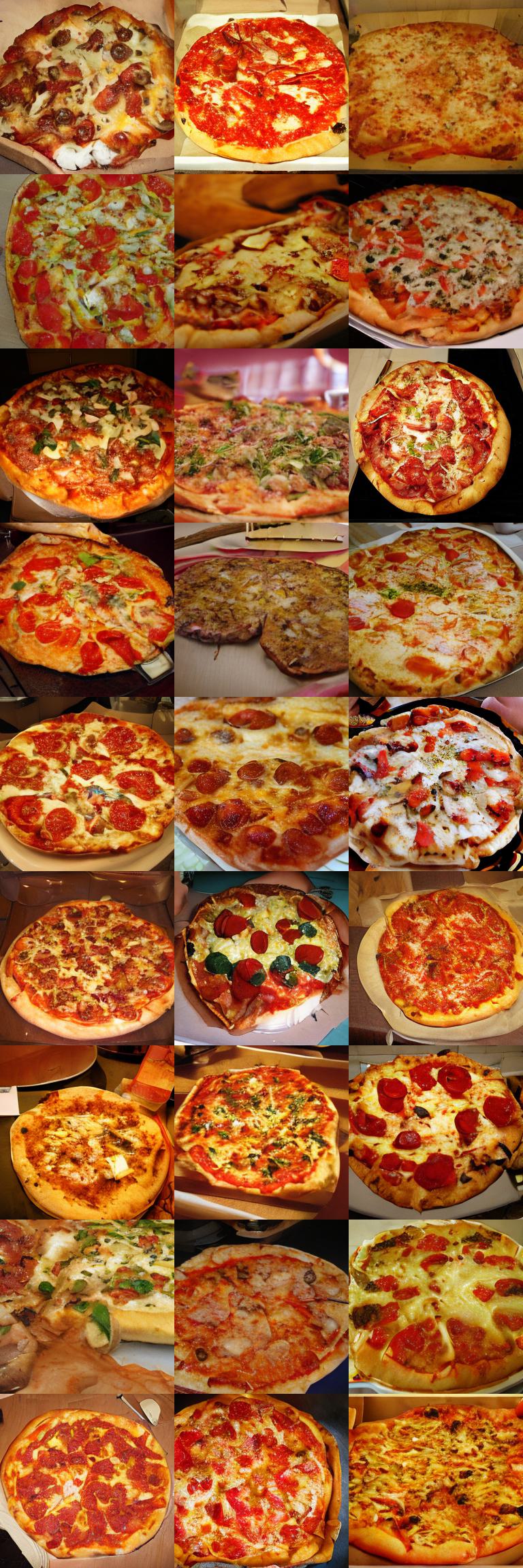}
}
\caption{Non-curated samples of our DiT-B/2 variant on ImageNet ($ \text{cfg}=2.0$).}
    \label{fig:imnet_cfg2.0c}
\vspace{-2mm}
\end{figure}

\begin{figure}[t]
\centering
\subfloat[Class: 975 - `Lakeside']{
    \includegraphics[width=0.48\linewidth]{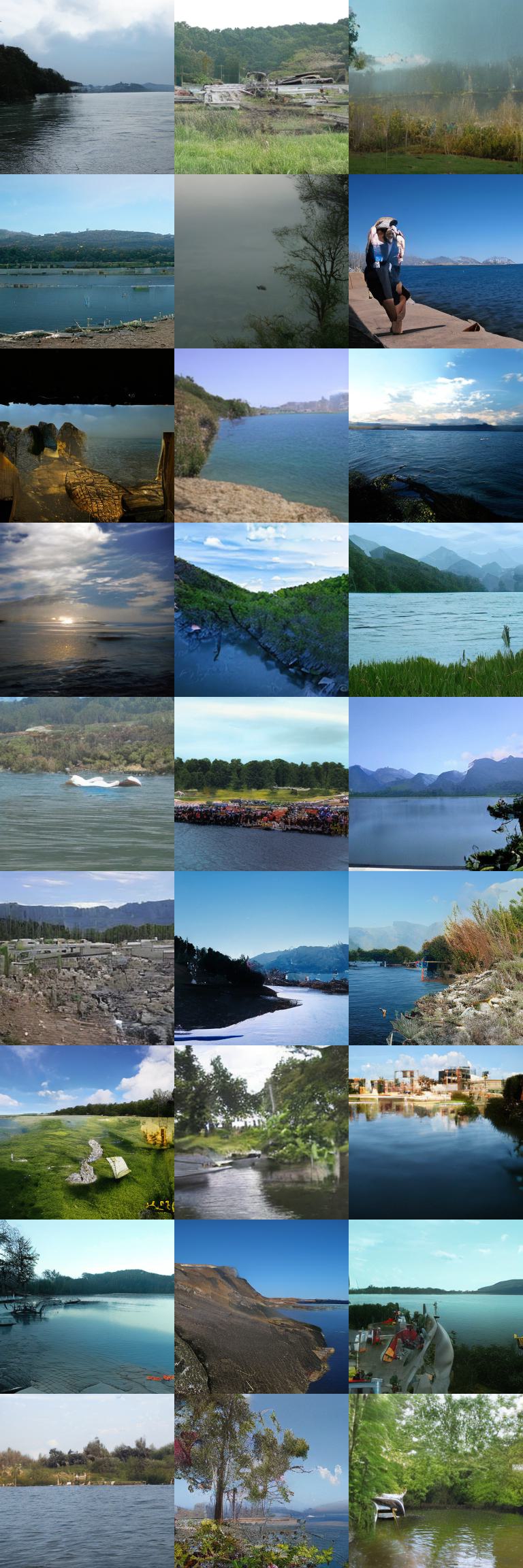}
}
~
\subfloat[Class: 984 - `Rapeseed']{
    \includegraphics[width=0.48\linewidth]{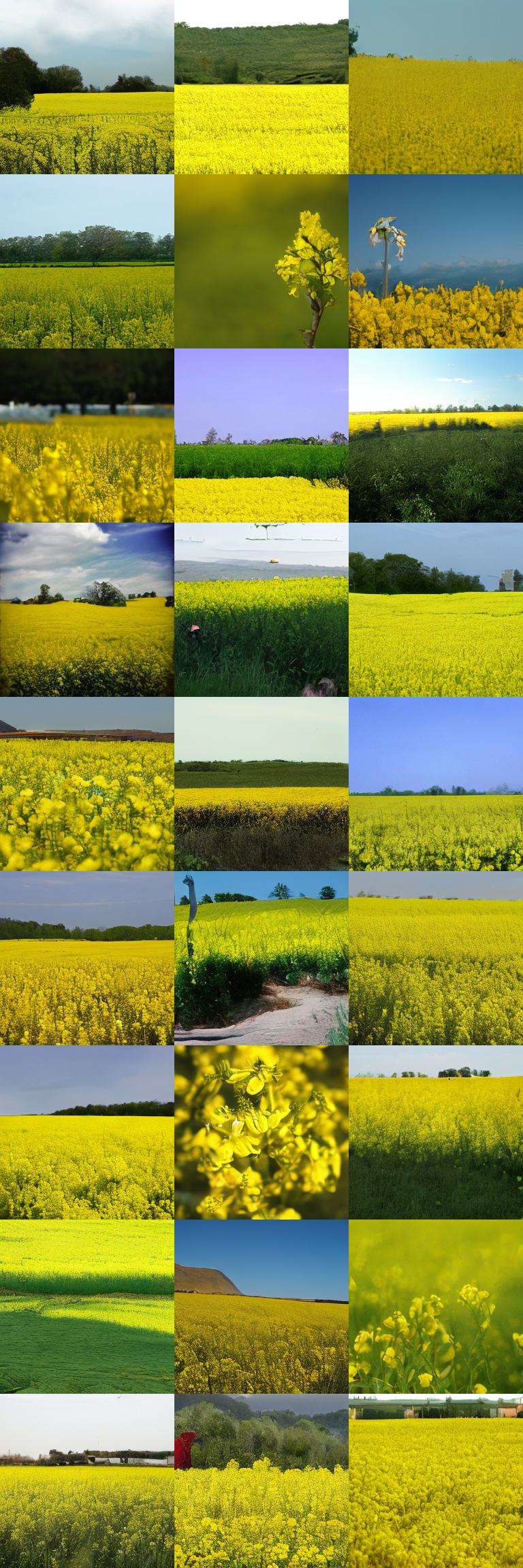}
}
\caption{Non-curated samples of our DiT-B/2 variant on ImageNet ($ \text{cfg}=1.5$).}
    \label{fig:imnet_cfg1.5}
\vspace{-2mm}
\end{figure}

\begin{figure}[t]
\centering
\subfloat[Class: 817 - `Sports car']{
    \includegraphics[width=0.48\linewidth]{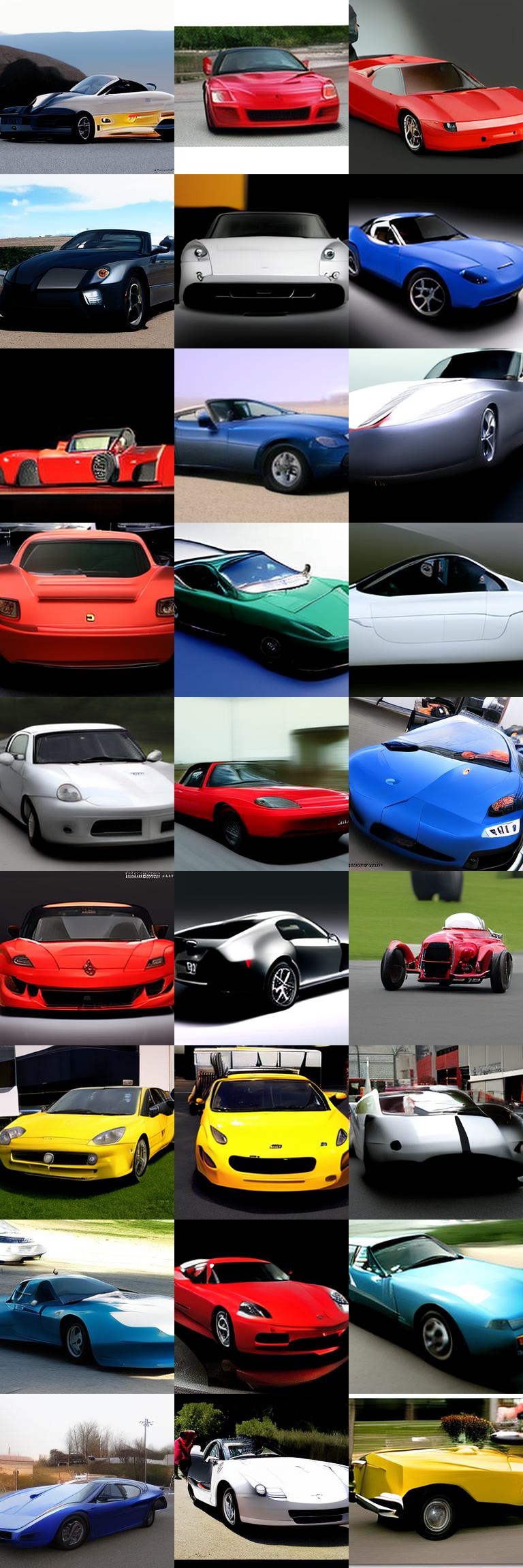}
}
~
\subfloat[Class: 825 - `Stone wall']{
    \includegraphics[width=0.48\linewidth]{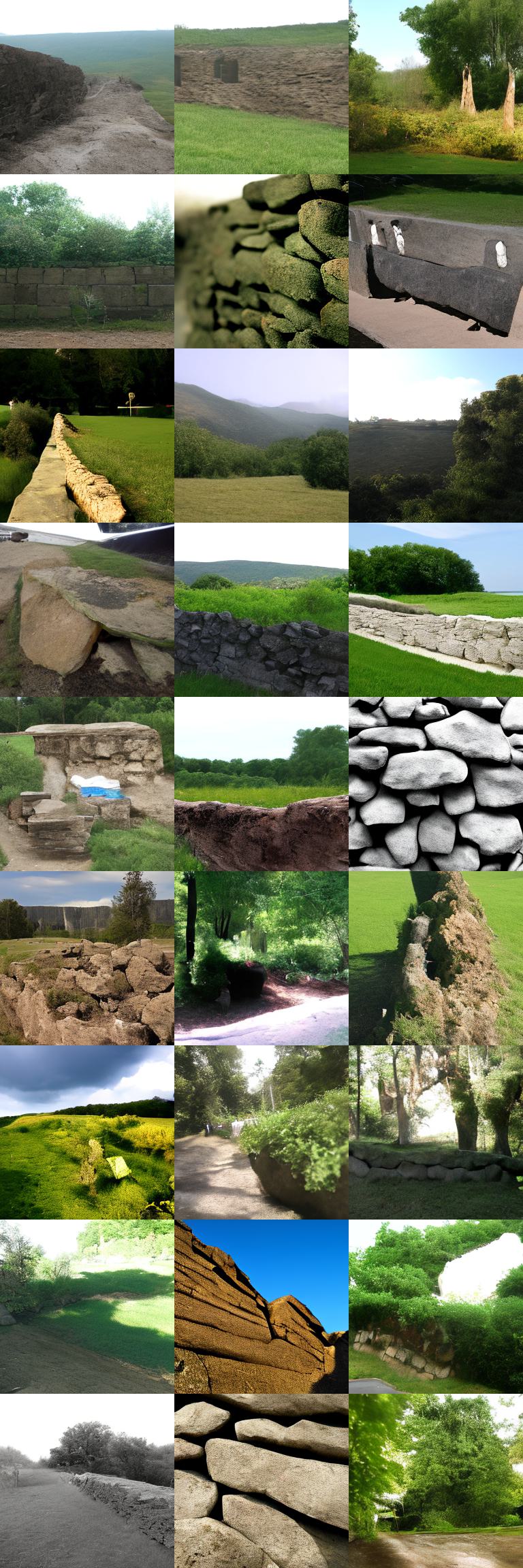}
}
\caption{Non-curated samples of our ADM variant on ImageNet ($ \text{cfg}=1.5$).}
    \label{fig:imnet_adm_cfg1.5}
\vspace{-2mm}
\end{figure}

% \subsection{Comparison of training convergence}

\subsection{More qualitative results}
For better illustration, additional visual examples are given for CelebA-HQ 256 in \cref{fig:more_celeb_samples}, FFHQ in \cref{fig:ffhq_samples}; LSUN Church in \cref{fig:church_samples}; LSUN Bedroom in \cref{fig:bed_samples}; CelebA-HQ 512 in \cref{fig:celeb512_samples}, and ImageNet in \cref{fig:imnet_cfg4.0_00,fig:imnet_cfg4.0_01,fig:imnet_cfg4.0_02,fig:imnet_cfg4.0_03,fig:imnet_cfg4.0_04}, \cref{fig:imnet_cfg2.0a,fig:imnet_cfg2.0b,fig:imnet_cfg2.0c}, and \cref{fig:imnet_cfg1.5,fig:imnet_adm_cfg1.5} with guidance scale of 4.0, 2.0, and 1.5, respectively.

\section{Implementation details}
\minisection{Network configuration.} 
We adopt two types of network architectures, namely UNet (ADM) and transformer network (DiT), for most experiments. \cref{tab:adm_config} shows detailed configurations of the ADM network on different datasets. For DiT, we adopt the DiT-L/2 variant for the unconditional generation with shape $256\times256$, and DiT-B/2 for class-conditional generation on ImageNet. Additionally, the size of DiT network is described at \cref{tab:network_size}. 

\begin{table}[t]
    \centering
    \caption{Size of DiT Network.}
    \begin{tabular}{l|c|c|c}
        \toprule
        Network & Image size & Params (M) & FLOPs (G) \\
        \toprule
        DiT-L/2 & 256 & 457 & 80.74 \\
        DiT-B/2 & 256 & 130 & 23.01 \\
        \bottomrule
    \end{tabular}
    \label{tab:network_size}
\end{table}

\begin{table*}[t]
\centering
\caption{ADM configurations.}
\begin{tabular}{l|c|c|c|c}
\toprule
                                 & \makecell{CelebA, FFHQ \\ \& Bed(256)} &  {Church(256)} &{CelebA(512)} &
                                 Imnet(256) \\ % &
                                 % {CelebA(1024)} \\
\midrule
\# of ResNet blocks per scale     & 2        & 2                & 2  & 2 \\% & 2  \\
Base channels                    & 256      & 256                 & 256 & 256 \\ % & 256     \\
Channel multiplier per scale & 1,2,3,4      & 1,2,3,4 & 1,2,2,2,4 & 1,2,3,4 \\ % & 1,1,2,2,4,4 \\
Attention resolutions              & 16, 8, 4        & 16,8                & 16, 8 & 16,8,4 \\ %  & 16, 8    \\
Channel multiplier for embeddings       & 4      & 4       & 4   & 4 \\ %  & 4         \\
Label dimensions & 0 & 0 & 0 & 1000 \\
\midrule
Params (M) & 368 & 356 & 352 & 369\\ 
FLOPs (G) & 28.9 & 28.9 & 94.2 & 28.9 \\

\bottomrule
\end{tabular}

\label{tab:adm_config}
\end{table*}

\minisection{Training hyper-params.} In \cref{tab:hyperparams_adm} and \cref{tab:hyperparams_dit}, we provide training hyperparameters for unconditional image generation on ADM and DiT networks, respectively. In both tables, each setting also shows the estimated training days.

\begin{table}[!ht]
    \centering
    \caption{Hyper-parameters of ADM network.}
    \begin{tabular}{l|c|c|c|c|c}
        \toprule
                                            & CelebA 256  & FFHQ    & CelebA 512 & Church \& Bed & IMNET \\
        \midrule
        $\text{lr}$                            & $5\text{e-5}$  & $2\text{e-}5$  & $5\text{e-}5$      & $5\text{e-}5$ & 1\text{e-}4   \\
        Adam optimizer ($\beta_1$ \& $\beta_2$) & 0.9, 0.999  & 0.9, 0.999  & 0.9, 0.999 & 0.9, 0.999 & 0.9, 0.999 \\
        % EMA                                 & 0.9999    & 0.9999    & 0.999       & 0.999       \\
        Batch size                          & 112       & 128       & 64        & 128   & 96   \\
        \# of epochs                        & 600      & 500       & 500     & 500  &  1200  \\
        \# of GPUs                          & 2         & 1         &  2            & 1    & 8  \\
        \# of training days &1.3  & 4.4 &5.9  & 6.9 \& 7.3 & 21.6 \\
        \bottomrule
    \end{tabular}
    \label{tab:hyperparams_adm}
\end{table}

\begin{table}[!ht]
    \centering
    \caption{Hyper-parameters of DiT network.}
    \begin{tabular}{l|c|c|c|c}
        \toprule
                                            & CelebA 256  & FFHQ & Church \& Bed & IMNET \\
        \midrule
        Model & DiT-L/2 & DiT-L/2 & DiT-L/2 & DiT-B/2 \\
         $\text{lr}$                            & $2\text{e-4}$  & $2\text{e-}4$ & $1\text{e-}4$ & $1\text{e-}4$ \\
        AdamW optimizer ($\beta_1$ \& $\beta_2$) & 0.9, 0.999  & 0.9, 0.999 & 0.9, 0.999 & 0.9, 0.999 \\
        % EMA                                 & 0.9999    & 0.9999    & 0.999       & 0.999       \\
        Batch size                          & 32       & 32       &  96 \& 32 & 160     \\
        \# of epochs                        & 500      & 500       & 500 & 900  \\
        \# of GPUs                          & 1         & 1       & 2  \& 1  & 8   \\
        \# of training days & 5.8 & 5.5 & 6.1 \& 12.5 & 12.0 \\
        \bottomrule
    \end{tabular}
    \label{tab:hyperparams_dit}
\end{table}

\end{document}